\documentclass[preprint,10pt,a4paper,3p]{elsarticle}

\usepackage{amsmath,amsthm,amssymb,graphicx,bm,bbm,physics,xcolor}
\usepackage[hidelinks]{hyperref}
\usepackage[mathcal]{euscript}
\usepackage{stfloats}
\usepackage[caption=false,font=normalsize,labelfont=sf,textfont=sf]{subfig}
\usepackage{pdfpages}

\def\h{\bm{h}}
\def\c{\bm{c}}

\def\xi{\bm{\xi}}
\def\pib{\bm{\varphi}}
\def\gammab{\bm{\gamma}}
\def\lambdab{\bm{\lambda}}
\def\psib{\bm{\psi}}
\def\pib{\bm{\pi}}

\DeclareMathOperator*{\sign}{sign}

\DeclareMathAlphabet{\mathdutchcal}{U}{dutchcal}{m}{n}
\DeclareMathOperator*{\argmin}{arg\min}

\newtheorem{theorem}{Theorem}

\newtheorem{corollary}{Corollary}
\newtheorem{inlemma}{Lemma}
\newtheorem{assump}{Assumption}
\newtheorem{property}{Property}

\allowdisplaybreaks

\usepackage{setspace}
\usepackage[numbers]{natbib}

\journal{Signal Processing}

\biboptions{sort&compress}
\begin{document}

\doublespacing

\begin{frontmatter}
    \title{Non-Asymptotic Performance of Social Machine Learning Under Limited Data \tnoteref{t1}}	
    \tnotetext[t1]{A preliminary short conference version of this work was previously published in \cite{hu2023performance}.}
    
    \author[1]{Ping~Hu\corref{cor1}}
    \ead{ping.hu@epfl.ch}
    
    \author[1]{Virginia~Bordignon}
    \ead{virginia.bordignon@epfl.ch}
    
    \author[1]{Mert~Kayaalp}
    \ead{mert.kayaalp@epfl.ch}
    
    \author[1]{Ali~H.~Sayed}
    \ead{ali.sayed@epfl.ch}
    
    \cortext[cor1]{Corresponding author}
    \affiliation[1]{organization={School of Engineering, EPFL},
    postcode={CH-1015},
    city={Lausanne},
    country={Switzerland}}

    \begin{abstract}
        This paper studies the probability of error associated with the social machine learning framework, which involves an independent training phase followed by a cooperative decision-making phase over a graph. This framework addresses the problem of classifying a stream of unlabeled data in a distributed manner. In this work, we examine the classification task with \emph{limited} observations during the decision-making phase, which requires a non-asymptotic performance analysis. We establish a condition for consistent training and derive an upper bound on the probability of error for classification. The results clarify the dependence on the statistical properties of the data and the combination policy used over the graph. They also establish the exponential decay of the probability of error with respect to the number of unlabeled samples. 
    \end{abstract}
        
    \begin{keyword}
        Social machine learning\sep probability of error\sep classification\sep non-asymptotic analysis
    \end{keyword}

\end{frontmatter}

\section{Introduction}
\label{sec:intro}
Social learning is a useful paradigm for addressing decision-making tasks involving a group of agents. Practical applications arise in various scenarios, such as detection and object recognition using autonomous robots, as well as statistical inference and learning across multiple processors  \cite{krishnamurthy2013social}. In this paper, we focus on the \emph{social machine learning} (SML) framework introduced in \cite{bordignon2023learning}, which is a \emph{data-driven} cooperative decision-making paradigm. The main motivation for the introduction of this framework is to address a critical limitation of traditional \emph{social learning} solutions \cite{bordignon2021adaptive,jadbabaie2012nonbayesian,shahrampour2016distributed,nedic2017fast,salami2017social,lalitha2018social,matta2020Interplay,kayaalp2022arithmetic}. These solutions allow a group of agents to interact over a graph to arrive at consensus decisions about a hypothesis of interest. However, a limiting assumption in all these studies is the requirement that the likelihood models for data generation are known beforehand. The SML strategy removes this requirement, thus opening up the door for solving classification tasks in a distributed manner with performance guarantees by relying solely on a data-driven implementation.

\begin{figure}[!t]
    \centering
    \includegraphics[width=.75\linewidth]{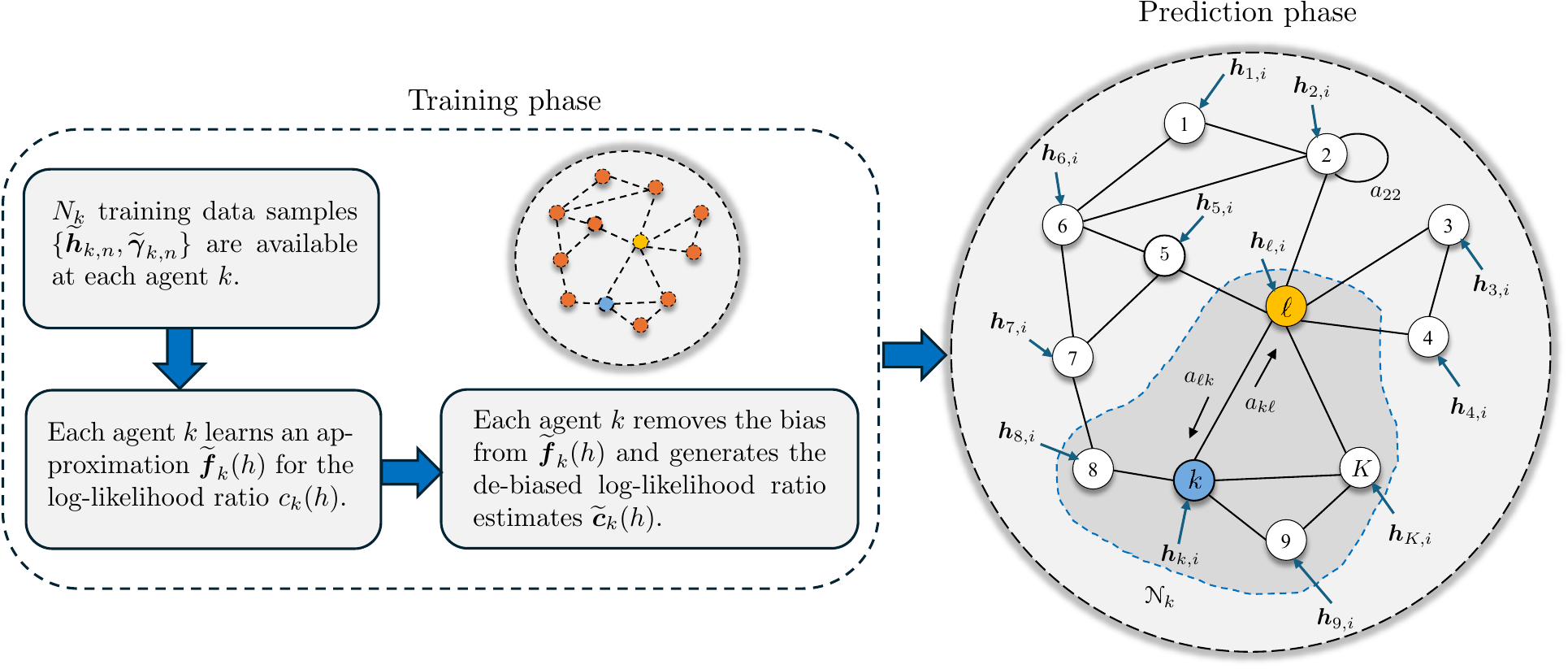}
    \caption{SML architecture. (\emph{Left panel}) The independent training process where each agent $k$ finds an optimal model $\widetilde{\bm{f}}_k$ based on its training set and constructs a classifier $\widetilde{\bm{c}}_k$ involving a debiasing operation. (\emph{Right panel}) The cooperative classification process where each agent $k$ receives a sequence of streaming observations $\h_{k,i}$ and implements a social learning protocol to enhance the prediction performance. The neighboring set $\mathcal{N}_k$ of agent $k$ is marked by the area highlighted in gray.}
    \label{fig: sml architecture}
\end{figure}

The SML strategy involves two learning phases, as depicted in Fig. \ref{fig: sml architecture}. In the \emph{training} phase on the left, each agent trains a classifier independently using a finite set of labeled samples within a supervised learning framework (such as logistic regression, neural networks, or other convenient frameworks). The purpose of this phase is to learn some discriminative information that allows agents to distinguish different hypotheses. The output of each trained classifier is used to form a local decision statistic for inference in the form of a log-likelihood ratio \cite{matta2018estimation,dedecius2018bayesian}. In the \emph{prediction} phase on the right, agents receive \emph{streaming} unlabeled samples and implement a social learning protocol based on the trained classifiers to infer the true state. With the well-established performance guarantees for both supervised learning and more recent social learning solutions, it is expected that the SML strategy, which combines the benefits of both approaches, should be able to deliver correct learning with high probability for a sufficient number of training samples. 

To support this claim, the work \cite{bordignon2023learning} has provided a rigorous theoretical analysis on the probability of consistent training concerning the \emph{asymptotic} truth learning in the prediction phase, and also illustrated the excellent classification performance of the SML strategy through extensive supporting simulations. As we will show in the main body of this paper, the probability of consistent training derived in \cite{bordignon2023learning} provides an upper bound for the probability of error in the \emph{infinite}-sample case, namely, the probability of correct decisions when the number of unlabeled samples grows indefinitely in the prediction phase. A series of interesting questions naturally emerge in this context. For instance, what would the learning performance of the SML strategy be in the \emph{finite}-sample case, which is common in practice? How much data do we need during both training and prediction to achieve a prescribed learning performance? This work answers these two questions by developing a \emph{non-asymptotic} performance analysis for the SML strategy. There are both practical and theoretical interests for this kind of investigation. First, collecting abundant samples is generally a time-consuming and expensive task, therefore learning with a \emph{limited} number of samples is desirable when possible. Second, the non-asymptotic performance analysis with respect to (w.r.t.) the number of samples provides a refined characterization of the learning behavior of the SML strategy and, in particular, leads to insights on the convergence rate of truth learning. However, as the arguments in this paper reveal, the derivations tend to be challenging but will ultimately lead to revealing insights. 

Specifically, in this paper, we will study the binary decision-making problem where the agents receive a finite number of unlabeled samples for inference. For clarity, we will refer to this scenario as the \emph{statistical classification} problem since the task is to classify a sequence of samples. 

{\bf \emph{Related works.}} The \emph{non-streaming} statistical classification problem where all unlabeled samples are provided at once by a testing sequence rather than arriving in a streaming manner, has been widely studied in the literature. One typical solution method is the type-based test that compares the closeness between the empirical distributions of the training sequence and the testing sequence \cite{gutman1989asymptotically,devroye2002note}. Due to the reliance on empirical distributions, the alphabet of the sources in these works is assumed to be either finite or the growth rate of the alphabet size is constrained \cite{kelly2012classification}. ML-based methods have also been employed for this non-streaming setting in \cite{braca2022statistical}, where a neural network is trained to approximate the decision statistics needed for log-likelihood ratio tests. The empirical mean of the decision statistic (i.e., the output of the trained classifier) associated with each sample in the testing sequence is used for classification. In this way, the constraint of a finite alphabet in the type-based methods is circumvented. All the above works \cite{gutman1989asymptotically,devroye2002note,kelly2012classification,braca2022statistical} focus on the single-agent scenario, where there is only one decision maker for classification. Different from them, the SML framework of this paper applies more broadly to \emph{multi}-agent decision-making problems in the \emph{streaming} setting. The combination of the qualifications ``multi-agent,''  ``streaming,'' and ``finite amount of data''  for decision making adds a layer of complexity that we show how to address in this work. 

\emph{Notation}: We use boldface fonts to denote random variables, and normal fonts for their realizations, e.g., $\bm{x}$ and $x$. $\mathbb{E}$ and $\mathbb{P}$ denote the expectation and probability operators, respectively. Since this work studies the operation of the learning process during both the \emph{training} and \emph{prediction} phases, we need to distinguish between variables appearing in both phases. We will use the tilde symbol $\sim$ to differentiate between the variables. For example, feature vectors used during training will be topped by the symbol $\sim$ and denoted by $\widetilde{h}$, while feature vectors used during prediction will be denoted simply by $h$ without the $\sim$ symbol. The same convention applies to all other variables in both stages of learning (training and prediction). A table summarizing the main symbols used in the paper is provided in the supplimentary material \cite{supp}.  

\section{The social machine learning (SML) strategy}
\label{sec:sml}

We first review the SML strategy and use this opportunity to familiarize the reader with the notation used for both stages of learning. We thus consider a network of $K$ connected agents or classifiers indexed by $k\in\mathcal{K}\triangleq\{1,2,\dots,K\}$, \label{notation: K} trying to solve a binary classification task. The agents are \emph{heterogeneous} in that their observations may follow different statistical models even when they are attributed to the same class. An example of this situation arises in multi-view learning \cite{zhao2017multi}, where each agent observes a different perspective of the same phenomenon and seeks to uncover the underlying state.  We denote the set of binary hypotheses by $\Gamma\triangleq\{+1,-1\}$. \label{notation: Gamma} In the independent training phase shown on the left of Fig. \ref{fig: sml architecture}, each agent $k$ has a set of $N_k$ labeled examples consisting of pairs $\{(\widetilde{\h}_{k,n},\widetilde{\gammab}_{k,n})\}_{n=1}^{N_k}$,\label{notation: training sample} where $\widetilde{\h}_{k,n}$ is the $n$-th feature vector and $\widetilde{\gammab}_{k,n}\in\Gamma$ is the corresponding label. The feature space of agent $k$ is denoted by $\mathcal{H}_k$\label{notation: feature space}. The pair $(\widetilde{\h}_{k,n},\widetilde{\gammab}_{k,n})$ is distributed according to 
\begin{equation}\label{eq: training sample distribution}
    (\widetilde{\h}_{k,n},\widetilde{\gammab}_{k,n})\sim\widetilde{p}_k(h,\gamma)
\end{equation}
which we factor according to Bayes rule into two equivalent forms
\begin{equation}
    \widetilde{p}_k(h,\gamma) = \widetilde{p}_k(h) \times  \widetilde{p}_{k}(\gamma|h)=\widetilde{p}_k(\gamma)\times L_k(h|\gamma)
\end{equation}
where we are denoting the marginal distributions of the label and feature data by $\widetilde{p}_k(\gamma)$ and $\widetilde{p}_k(h)$, and the conditional distributions by $\widetilde{p}_k(\gamma|h)$ and $L_k(h|\gamma)$. We refer to this last distribution  as the \emph{likelihood model}: it explains how the feature data are generated for each label. This likelihood function (i.e., the generative model) is \emph{unknown} and we are not using a tilde symbol on top of it because we assume that the same generative model applies during training and prediction, which is a natural assumption for the statistical classification problem. Although unnecessary, we assume that the training set is balanced, i.e., $\widetilde{p}_k(\gamma)=\frac{1}{2}$. The purpose of the training phase is to learn the likelihood functions, which can then be used during the prediction phase to classify new feature vectors. More specifically, for classification purposes, it is sufficient to learn the log-ratio of the posterior probabilities (also known as the logit function or the log-odds \cite{bishop2006pattern}), namely, the quantity
\begin{equation}\label{eq: logit function}
    c_k(h)\triangleq\log\frac{\widetilde{p}_k(+1|h)}{\widetilde{p}_k(-1|h)}.
\end{equation}
Once learned, this decision statistic can be used by agent $k$ to label feature vector $h$. For example, we know that the optimal Bayes classifier would assign $h$ to class $\gamma=+1$ if $c_k(h)$ is positive and to class $\gamma=-1$ otherwise. Under the uniform prior condition, $\widetilde{p}_k(\gamma)=\frac{1}{2}$, the above log-odds reduces to the log-likelihood ratio:
\begin{equation}\label{eq: log-likelihood}
    c_k(h)=\log\frac{L_k(h|+1)}{L_k(h|-1)}.
\end{equation}
This confirms that it is sufficient for the SML strategy to compute  the log-likelihood ratio $c_k(h)$ in \eqref{eq: log-likelihood} to be able to perform classification. However, since the likelihood models are unknown, $c_k(h)$ is not available and will need to be learned. We explain next how this can be done by reviewing briefly the construction from \cite{bordignon2023learning} for the training phase.

\subsection{Training phase}
During the training phase, each agent will focus on learning $c_k(h)$ by approximating the posterior probabilities for each class, denoted by $\widehat{p}_k(+1|h)$ and $\widehat{p}_k(-1|h)$, and using \eqref{eq: logit function}. For example, each agent could use its labeled samples to train a logistic classifier or a neural network. The output of the classifier used at this agent could then be utilized to approximate $c_k(h)$. We denote this approximation
\begin{equation}\label{eq: approximate logit function f_k}
    f_k(h)\triangleq\log\frac{\widehat{p}_k(+1|h)}{\widehat{p}_k(-1|h)} \qquad \qquad (\text{initial approximation for $c_k(h)$})
\end{equation} 
where the function $f_k$ belongs to some admissible class $\mathcal{F}_k:\mathcal{H}_k\mapsto\mathbb{R}$\label{notation: function class}. The form of the class $\mathcal{F}_k$ will depend on the type of classifiers used by each agent. For example, when agent $k$ trains a logistic classifier with parameter $w$, the class ${\cal F}_k$ will consist of linear functions parameterized by $w$, i.e., $f_k(h)=w^\top h$. In principle, the quantity $f_k(h)$ obtained in this manner serves as an approximation for the log-odds \eqref{eq: logit function} or \eqref{eq: log-likelihood}. We could have denoted it by $\widehat{c}_k(h)$. However, as we proceed to explain, the function $f_k(h)$ will need to be centered by subtracting a bias term from it, after which we will be able to obtain the desired estimate for $c_k(h)$ --- see the construction \eqref{eq: trained classifier} below. First, we explain how to determine $f_k(h)$. By definition \eqref{eq: approximate logit function f_k}, $f_k$ is a function of the output of the classifier used by agent $k$. In order to learn parameters of the classifier that ultimately determine the best function $\widetilde{\bm{f}}_k$, it is common to minimize some empirical risk function based on the training data, such as
\begin{equation}\label{eq: trained f_k}
    \widetilde{\bm{f}}_k\triangleq\argmin_{f_k\in\mathcal{F}_k}\widetilde{\bm{R}}_{k,{\rm emp}}(f_k),
\end{equation}
where $\widetilde{\bm{R}}_{k,{\rm emp}}(f_k)$ is the risk function used during training by agent $k$ and defined in terms of some loss function $\Phi(\cdot)$ applied to the training samples, namely, 
\begin{equation}\label{eq: empirical risk}
    \widetilde{\bm{R}}_{k,{\rm emp}}(f_k)\triangleq\frac{1}{N_k}\sum_{n=1}^{N_k}\Phi\left({\widetilde{\gammab}_{k,n}f_k(\widetilde{\h}_{k,n})}\right).
\end{equation}
In this formulation, the function $f_k$ applied to the feature vector $\widetilde{\h}_{k,n}$ generates an estimate for the true label $\widetilde{\gammab}_{k,n}$. As is common for many loss functions in learning theory, the loss tends to be a function of the product of the true label and its estimate, as seen in the argument of $\Phi(\cdot)$ in \eqref{eq: empirical risk}. This form holds for logistic classifiers as well as for softmax constructions in neural networks. We impose the following standard condition on the loss function $\Phi(\cdot)$.
\begin{assump}[\textbf{Conditions on the loss function}]\label{assump: risk function}
    The loss function $\Phi(\cdot):\mathbb{R}\mapsto\mathbb{R}_{+}$ is convex, non-increasing and differentiable at $0$ with $\Phi^\prime(0)<0$. Also, it is $L_{\Phi}$-Lipschitz. \qed
\end{assump}
\noindent This assumption guarantees that the loss function $\Phi$ is \emph{classification-calibrated} (see Theorem 2 in \cite{bartlett2006convexity}). And therefore, the resulting function $\widetilde{\bm{f}}_k$ will be Bayes-risk optimal,  meaning that the sign of $\widetilde{\bm{f}}_k(h)$ and $c_k(h)$ will be the same for any feature vector $h\in\mathcal{H}_k$ when the number of training samples $N_k$ goes to infinity. Due to this useful property, classification-calibrated loss functions are extensively utilized for classification tasks in the literature. The work \cite{bordignon2023learning} focused solely on the logistic loss $\Phi(x)=\log(1+e^{-x})$, which satisfies all the conditions of Assumption \ref{assump: risk function} with $L_\Phi = 1$. However, in addition to the logistic loss, Assumption \ref{assump: risk function} covers many other widely-used loss functions, including the exponential loss $\Phi(x)=e^{-x}$ used in boosting algorithms \cite{freund1997decision,friedman2000additive}, and the hinge loss $\Phi(x)=\max(0,1-x)$ used in support vector machines \cite{cortes1995support}. In this paper, we will consider generic loss functions $\Phi(\cdot)$ under Assumption \ref{assump: risk function} and will not be limited to the logistic loss.  We also consider more general classifier structures that belong to a general class of \emph{bounded} real-valued functions $\mathcal{F}_k$. The following assumption on the function $f_k$ is imposed.
\begin{assump}[\textbf{Boundedness of functions}]\label{assump: bound}
    There exists a constant $\beta>0$ such that $\forall h\in\mathcal{H}_k$, $\abs{f_k(h)}\leq\beta$ for each agent $k\in\mathcal{K}$ and each function $f_k\in\mathcal{F}_k$.  \qed
\end{assump}
\noindent After the training phase, the learned function $\widetilde{\bm{f}}_k$ resulting from solving \eqref{eq: trained f_k} at agent $k$ (e.g., by using a stochastic gradient algorithm or backpropagation) can turn out to be biased. This means the following. If we refer to \eqref{eq: logit function}, we expect the log-odds to be positive when the true label is $+1$ and negative otherwise. In other words, if evaluated over many feature vectors $h\in{\cal H}_k$, we expect the values of $c_k(h)$ to be more or less uniformly distributed between positive and negative values. Due to biases and distortions introduced during training, this ``uniform'' split may be perturbed. One useful step is to center the learned log-odds by removing its mean. For this reason, the ultimate estimate for $c_k(h)$ is the following centered quantity, which we now denote by $\widetilde{\c}_k(h)$: 
\begin{equation}\label{eq: trained classifier}
    \widetilde{\bm{c}}_k(h)\triangleq\widetilde{\bm{f}}_k(h)-\widetilde{\bm{\mu}}_k(\widetilde{\bm{f}}_k)\qquad \qquad (\text{\rm ultimate approximation for $c_k(h)$})
\end{equation} 
where $\widetilde{\bm{\mu}}_k(\widetilde{\bm{f}}_k)$ is the \emph{empirical training mean} calculated from the training data:
\begin{equation}\label{eq: empirical training mean}
    \widetilde{\bm{\mu}}_k(f_k)\triangleq\frac{1}{N_k}\sum_{n=1}^{N_k}f_k(\widetilde{\h}_{k,n}),\quad \forall f_k\in\mathcal{F}_k.
\end{equation}
Discounting the empirical training mean was already suggested in \cite{bordignon2023learning} as a \emph{debiasing} operation to mitigate possible biased models resulting from the training process. In summary, the objective of the training phase is to use the training data $(\widetilde{\h}_{k,n},\widetilde{\gammab}_{k,n})$ to arrive at the classifier models $\{\widetilde{\c}_k\}$; one for each agent. This construction is shown on the left part of Fig. \ref{fig: sml architecture}. Once the models $\{\widetilde{\c}_k\}$ are learned during training, they are frozen and will be used subsequently during the prediction phase. 

\subsection{Prediction phase}
In the prediction (testing) phase, the dispersed agents work jointly to solve a binary classification problem using their learned models. Specifically, at each time $i$, each agent $k$ receives a new feature vector $\h_{k,i}\in\mathcal{H}_k$, so that each agent $k$ has access to a growing stream of feature vectors $\bm{h}_{k,1}, \bm{h}_{k,2}, \dots$, which are identically and independently distributed (i.i.d.) according to some unknown likelihood model $L_k(\cdot|\gammab_0)$, where $\gammab_0\in\Gamma$ is the \emph{true unknown state} in force during the prediction phase. The objective of the prediction phase is to determine whether $\gammab_0=+1$ or $\gammab_0=-1$ for the streaming data. Let
    \begin{equation}\label{eq: random vector of observations}
        \bm{\mathsf{h}}_{i}\triangleq\text{col}\big\{\h_{1,i},\h_{2,i},\dots,\h_{K,i} \big\}
    \end{equation}
denote the collection of observations received by all agents at time $i$. This is the snapshot of the input to the graph at that time instant. Then, $\bm{\mathsf{h}}_i$ is an i.i.d. vector taking values in $\prod_{k=1}^{K}\mathcal{H}_k$ and distributed as $\prod_{k=1}^K L_k(\cdot|\gammab_0)$. To solve the classification problem under infinite samples, reference \cite{bordignon2023learning} devised a distributed learning rule inspired by the social learning studies \cite{shahrampour2016distributed,nedic2017fast,salami2017social,lalitha2018social,matta2020Interplay,kayaalp2022arithmetic}. We briefly describe the basic framework here.

Each agent $k$ is assumed to have access to some likelihood models $\{L_k(h|\gamma)\}$, for each $\gamma\in\Gamma$. These functions model the agent's assumption of how the observations are generated. Actually, as the argument will show, the agents do not need to know the individual likelihoods $L_k(h|+1)$ and $L_k(h|-1)$ but only their ratio. And as already explained in the previous section, this ratio was learned during the training phase and represented by the quantity $\widetilde{\c}_k(h)$. We describe the social learning algorithm first by assuming knowledge of the individual likelihood functions $L_k(h|\gamma)$, but soon thereafter explain how they can be replaced by their estimated ratio and continue the discussion from there.

The standard social learning rule relies on iterative \emph{adaptation} and \emph{combination} steps. Let $\pib_{k,i}$ denote the belief vector of agent $k$ at time $i$, which is a probability mass function over the set of hypotheses $\Gamma$. In other words, $\bm{\pi}_{k,i}$ is a $2\times 1$ vector and each of its entries represents the confidence that agent $k$ has at time $i$ about whether the true label is $\gammab_0=+1$ or $\gammab_0=-1$. We will use the notation $\bm{\pi}_{k,i}(\gamma)$ to index the 2 entries of the belief vector. In the adaptation step, each agent uses the Bayes rule to incorporate information about the new observation $\h_{k,i}$ into the agent's \emph{intermediate} belief vector denoted by $\psib_{k,i}$:
\begin{equation}\label{eq: SL-adaptation}
    \psib_{k,i}(\gamma)=\frac{\pib_{k,i-1}(\gamma)L_k(\h_{k,i}|\gamma)}{\sum_{\gamma^\prime \in \Gamma} \pib_{k,i-1}(\gamma^\prime)L_k(\h_{k,i}|\gamma^\prime)},\quad \forall \gamma\in\Gamma=\{+1,-1\}.
\end{equation}
In the subsequent combination step, each agent aggregates the  intermediate beliefs from its neighbors using a certain pooling protocol. One representative protocol is the geometric rule described by
\begin{equation}\label{eq: SL-combination}
    \pib_{k,i}(\gamma)=\frac{\exp{\sum_{\ell\in\mathcal{N}_k} a_{\ell k} \log \psib_{\ell,i}(\gamma)}}{\sum_{\gamma^\prime\in\Gamma} \exp{\sum_{\ell\in\mathcal{N}_k} a_{\ell k} \log \psib_{\ell,i}(\gamma^\prime)}}
\end{equation}
for each $\gamma\in\Gamma$. The combination weights $a_{\ell k}$ that agent $k$ assigns to its neighbors $\ell$ satisfy: $\sum_{\ell=1}^K a_{\ell k}=1$, and $a_{\ell k}> 0$ for all $\ell\in \mathcal{N}_k$. In particular, it holds that $a_{\ell k}=0$ if $\ell\notin\mathcal{N}_k$, where $\mathcal{N}_k$ denotes the neighboring set of agent $k$ (see Fig. \ref{fig: sml architecture}). We introduce the following assumption on the network topology.
\begin{assump}[\textbf{Strongly-connected graph}]\label{assump: network}
    The underlying graph of the network is strongly connected. That is, there exist paths with positive combination weights between any two distinct agents in both directions (these trajectories need not be the same), and at least one agent has a self-loop, i.e., $a_{mm}>0$ for some agent $m$  \cite{sayed2014adaptation}. \qed
\end{assump}
\noindent Under this assumption and from the Perron-Frobenius theorem \cite{sayed2014adaptation,sayed2022inference},  we know that the $K\times K$ combination matrix $A=[a_{\ell k}]$ is primitive and has a Perron eigenvector $p$ satisfying:
\begin{equation}\label{eq: Perron eigenvector}
    Ap=p,\quad \sum_{k=1}^{K}p_k=1,\quad p_k>0, \quad \forall k\in\mathcal{K}.
\end{equation}
Furthermore, the second largest-magnitude eigenvalue of $A$, denoted by $\sigma$, is strictly smaller than 1. To see how the individual likelihoods can be replaced by their ratio, we remark that the social learning rule \eqref{eq: SL-adaptation}--\eqref{eq: SL-combination} can be written in a compact form by introducing the log-belief and log-likelihood ratios between labels $+1$ and $-1$, denoted by:
\begin{equation}\label{eq: log-belief ratio}
    \lambdab_{k,i} \triangleq\log\frac{\pib_{k,i}(+1)}{\pib_{k,i}(-1)},\quad c_k(\h_{k,i})\triangleq\log\frac{L_k(\h_{k,i}|+1)}{L_k(\h_{k,i}|-1)}.
\end{equation}
It is clear that the sign of $\lambdab_{k,i}$ represents the preference of agent $k$ for classification at time $i$: label $+1$ (or $-1$) will be selected if $\lambdab_{k,i}$ is positive (or negative). Combining \eqref{eq: SL-adaptation} and \eqref{eq: SL-combination}, we find that $\lambdab_{k,i}$ evolves according to the recursion
\begin{equation}\label{eq: SL-log-belief-ratio}
    \lambdab_{k,i}= \sum_{\ell=1}^{K} a_{\ell k}\big(\lambdab_{\ell,i-1}+c_{\ell}({\h}_{\ell,i})\big).
\end{equation}
As we explained before, the log-likelihood ratio is learned during the training phase. By replacing $c_{\ell}({\h}_{\ell,i})$ with the estimate $\widetilde{\bm{c}}_{\ell}(\h_{\ell,i})$ from \eqref{eq: trained classifier}, the following social learning rule is therefore employed in the prediction phase (where we continue to use the $\lambdab$ notation to avoid an explosion in symbols):
\begin{equation}\label{eq: sl learning rule}
    \lambdab_{k,i}= \sum_{\ell=1}^{K} a_{\ell k}\big(\lambdab_{\ell,i-1}+\widetilde{\bm{c}}_{\ell}({\h}_{\ell,i})\big).
\end{equation}
One notable feature of the learning rule \eqref{eq: sl learning rule} is that the information from the local observations is aggregated over both space (through $\ell$) and time (through $i$), which strengthens the decision-making capabilities of the agents. We provide an algorithmic description of the SML strategy in the supplementary material \cite{supp}. An important question now is to examine the performance guarantees of strategy \eqref{eq: sl learning rule} under the challenging constraint of the error introduced due to replacing the true log-likelihood ratio $\c_{\ell}$ by its approximation $\widetilde{\c}_{\ell}$.

\section{Consistency of the SML strategy}
\label{sec: consistency}
To being with, we recall that in \cite{bordignon2023learning} the SML strategy \eqref{eq: sl learning rule} was said to be \emph{consistent} if asymptotic truth learning in the prediction phase is attained, namely, if the true state $\gammab_0$ is learned by all agents when the number of observations (i.e., feature vectors $\{\h_{k,i}\}$) goes to \emph{infinity}. To examine consistency, we first review an important conclusion about the social learning rule \eqref{eq: sl learning rule} provided in  \cite{lalitha2018social,bordignon2023learning} in the asymptotic regime, namely that\footnote{This condition holds due to Assumption \ref{assump: bound}, which ensures that the log-likelihood ratio estimates $\{\widetilde{\bm{c}}_k(\h_{k,i})\}$ are bounded.}
\begin{equation}\label{eq: asymptotic convergence}
    \frac{1}{i}\lambdab_{k,i}\overset{\text{a.s.}}{\longrightarrow}\sum_{\ell=1}^Kp_{\ell}\mathbb{E}_{\gammab_{0}}\widetilde{\bm{c}}_{\ell}(\h_{\ell,i})\triangleq \widehat{\lambdab}_{\rm asym}
\end{equation}
where ``a.s.'' means almost sure convergence and $\mathbb{E}_{\gammab_{0}}$ denotes the expectation operator w.r.t. the true but unknown likelihoods $L_\ell(\cdot|\gammab_{0})$. We refer to the quantity on the right-hand side as an \emph{asymptotic decision statistic} and denote it by $\widehat{\lambdab}_{\rm asym}$. Then, the SML strategy is consistent when $\gammab_{0}\widehat{\lambdab}_{\rm asym}>0$ is satisfied. That is, $\widehat{\lambdab}_{\rm asym}>0$ if $\gammab_{0}=+1$, and $\widehat{\lambdab}_{\rm asym}<0$ otherwise. This notion of consistency is important because it suggests a construction to classify the feature vectors by examining the sign of $\widehat{\lambdab}_{\rm asym}$ or, alternatively, by examining a sufficient  condition described next in \eqref{eq: consistent training}.

To avoid confusion, it is worth noting that the trained models $\{\widetilde{\bm{f}}_k\}$ and classifiers $\{\widetilde{\bm{c}}_k\}$ are generated in the training phase, so they are random w.r.t. the training set. Therefore, following the training phase, we can ``freeze'' these quantities, which become deterministic and denoted by the realizations of $\{\widetilde{f}_k\}$ and $\{\widetilde{c}_k\}$. Let $\mu_k^{+}(f_k)$ and $\mu_k^{-}(f_k)$ denote the conditional means of a function $f_k\in\mathcal{F}_k$ under the two classes:
\begin{subequations}
    \begin{equation}\label{eq: expectation +1}
        \mu_k^{+}(f_k)\triangleq \mathbb{E}_{+1}f_k(\h_{k,i})=\mathbb{E}_{\h_{k,i}\sim L_k(\cdot|+1)}f_k(\h_{k,i}),
    \end{equation}
    \begin{equation}\label{eq: expectation -1}
        \mu_k^{-}(f_k)\triangleq \mathbb{E}_{-1}f_k(\h_{k,i})=\mathbb{E}_{\h_{k,i}\sim L_k(\cdot|-1)}f_k(\h_{k,i}).
    \end{equation}
\end{subequations}
As explained before, we expect the approximated log-likelihood ratio $f_k$ in \eqref{eq: approximate logit function f_k} to be positive when the class is $+1$ and negative otherwise. Equations \eqref{eq: expectation +1}--\eqref{eq: expectation -1} are computing the expected value for this ratio  for each agent under both classes, $+1$ and $-1$. We can combine the averages across the agents and compute the \emph{network average} weighted by the Perron entries:
\begin{equation}\label{eq: network expectation}
    \mu^{+}(f)\triangleq\sum_{k=1}^{K}p_k\mu_k^{+}(f_k),\quad
    \mu^{-}(f)\triangleq\sum_{k=1}^{K}p_k\mu_k^{-}(f_k),
\end{equation}
where the argument $f$ on the left-hand side refers to the dependence of the network averages $\mu^{+}(\cdot)$ and $\mu^{-}(\cdot)$ on the collection of functions $\{f_k\}$, i.e., $\mu^{+}(f)=\mu^{+}(f_1,\dots,f_K)$ and $\mu^{-}(f)=\mu^{-}(f_1,\dots,f_K)$. Similar notation will be used for other network quantities in this paper. Based on the social learning rule \eqref{eq: sl learning rule} and its convergence property in \eqref{eq: asymptotic convergence}, the following \emph{sufficient} condition for consistency is established in \cite{bordignon2023learning}:
\begin{equation}\label{eq: consistent training}
    \mu^{+}(\widetilde{f})>\widetilde{ \mu}({\widetilde{f}})\quad\text{and}\quad \mu^{-}(\widetilde{f})<\widetilde{ \mu}(\widetilde{ f})
\end{equation} 
where 
\begin{equation}\label{eq: network training mean}
    \widetilde{\mu}(\widetilde{f})\triangleq\sum_{k=1}^{K}p_k \widetilde{\mu}_k(\widetilde{f}_k)
\end{equation}
is the network average of the empirical training means specified in \eqref{eq: empirical training mean}. Condition \eqref{eq: consistent training} states that as long as the mean of the learned classifiers under class +1 is larger than the mean under class $-1$, and the empirical training mean of these classifiers sits in between, then the SML strategy should be able to classify correctly. To understand this condition, we note that by combining the convergence result \eqref{eq: asymptotic convergence} with the definitions \eqref{eq: trained classifier}, \eqref{eq: network expectation}, and \eqref{eq: network training mean}, we have
\begin{equation}\label{eq: convergence under different hypotheses}
    \widehat{\lambdab}_{\rm asym}{=}\begin{cases}
        \mu^{+}(\widetilde{{f}})-\widetilde{{\mu}}(\widetilde{{f}}), & \gammab_{0}=+1\\
        \mu^{-}(\widetilde{{f}})-\widetilde{{\mu}}(\widetilde{{f}}), & \gammab_{0}=-1
    \end{cases}
\end{equation}
and therefore $\gammab_0 \widehat{\lambdab}_{\rm asym} >0$ under condition \eqref{eq: consistent training}. This condition is deemed to be sufficient because the true state $\gammab_0$ remains fixed during the prediction phase. Thus, only the condition $\mu^{+}(\widetilde{f}) > \widetilde{\mu}(\widetilde{f})$ (or $\mu^{-}(\widetilde{f}) < \widetilde{\mu}(\widetilde{f})$) is necessary for the SML strategy to be consistent if $\gammab_0 = +1$ (or $\gammab_0 = -1$).
Now since the description \eqref{eq: consistent training} involves a set of trained models $\{\widetilde{f}_k\}$, the conditions in \eqref{eq: consistent training} depend on the randomness stemming from the training phase. For this reason, these conditions will be referred to as the condition for \emph{consistent training}, namely, for the training set to produce a consistent classifier for the subsequent prediction phase at the end of the training phase. The \emph{probability of consistent training}, denoted by $P_c$, is then defined as follows:
\begin{equation}\label{eq: probability of consistent training}
    P_c\triangleq\mathbb{P}\left(\mu^{+}(\widetilde{\bm{f}})>\widetilde{\bm \mu}({\widetilde{\bm f}}), \mu^{-}(\widetilde{\bm{f}})<\widetilde{\bm \mu}(\widetilde{\bm f})\right)
\end{equation}
where the boldface fonts are used to highlight the randomness in the training phase. Using similar techniques to the ones used in \cite{bordignon2023learning} which considered only the logistic loss, we will show in Theorem \ref{theorem: 1} that $P_c$ is lower bounded by a constant related to the number of training samples $N_k$, the Perron eigenvector $p$ of the combination matrix $A$, the properties of the loss function $\Phi$, and the Rademarcher complexity of the function class $\mathcal{F}_k$. For that purpose, we need to introduce some notation as follows. First, we define the \emph{target risk} during training at every agent $k$ as
\begin{equation}\label{eq: individual target risk}
    \widetilde{\mathsf{R}}_k^o\triangleq\inf_{f_k\in\mathcal{F}_k} \widetilde{R}_k(f_k)
\end{equation}
where $\widetilde{R}_k(f_k)$ is the expected (rather than empirical) risk associated with $f_k$---compare with expression \eqref{eq: empirical risk}:
\begin{equation}\label{eq: expected risk of f_k}
    \widetilde{R}_k(f_k)\triangleq \mathbb{E}_{(\widetilde{\h}_{k,n},\widetilde{\gammab}_{k,n})}\Phi\left(\widetilde{\gammab}_{k,n}f_k(\widetilde{\h}_{k,n})\right).
\end{equation}
Here, the expectation operator is w.r.t. the unknown distribution $\widetilde{p}_k(h,\gamma)$ of the training data $(\widetilde{\h}_{k,n},\widetilde{\gammab}_{k,n})$ described in \eqref{eq: training sample distribution}. The weighted network average of target risks is defined according to
\begin{equation}\label{eq: target risk}
    \mathsf{R}^o\triangleq\sum_{k=1}^Kp_k \widetilde{\mathsf{R}}_k^o.
\end{equation}
We will assume that the network average target risk $\mathsf{R}^o$ is strictly smaller than $\Phi(0)$, which in view of \eqref{eq: expected risk of f_k}, is the risk corresponding to the \emph{uninformative} classifier $f_k=0$. Formally, we assume
\begin{equation}\label{eq: R^o condition}
    \mathsf{R}^o<\Phi(0).
\end{equation}
This condition eliminates having $f_k=0$ as the optimal solution for all $k\in\mathcal{K}$. To understand this condition, we recall the definition of $f_k$ from \eqref{eq: approximate logit function f_k}. Suppose that $f_k(h)=0$ for all $h\in\mathcal{H}_k$, then it will hold that
\begin{equation}
    \widehat{p}_k(\gamma|h)=\frac{1}{2},  \text{ for any } h\in\mathcal{H}_k \text{ and }\gamma\in\Gamma.
\end{equation}
Consequently, agent $k$ will make an \emph{uninformed} decision that randomly assigns the labels $+1$ and $-1$ with equal probability to all feature vectors. In other words, this agent fails to learn any discriminative information for the two classes during the training phase. In comparison, according to definition \eqref{eq: target risk}, condition \eqref{eq: R^o condition} implies that there exists some agent $k$ such that $\widetilde{\mathsf{R}}^o_k<\Phi(0)$. That is, the uninformative classifier $f_k=0$ is not optimal for this agent. Therefore, condition \eqref{eq: R^o condition} ensures that there is \emph{at least one} agent that is able to make a better decision than random guessing.

Next, we introduce some notation related to the \emph{Rademacher complexity} of function classes $\{\mathcal{F}_k\}$ \cite{sayed2022inference,boucheron2005theory}. Let
\begin{equation}
    \widetilde{h}^{(k)}\triangleq\big\{ \widetilde{h}_{k,1},\dots,\widetilde{h}_{k,N_k} \big\}
\end{equation} 
be a fixed sample set of size $N_k$ for agent $k$. Then, the \emph{empirical} Rademacher complexity at agent $k$ for the sample set $\widetilde{h}^{(k)}$ is defined as
\begin{equation}\label{eq: individual empirical Rademacher complexity}
    \widetilde{\mathcal{R}}\Big(\mathcal{F}_k(\widetilde{h}^{(k)})\Big)\triangleq\mathbb{E}_{\bm{r}}\sup_{f_k\in\mathcal{F}_k}\abs{\frac{1}{N_k}\sum_{n=1}^{N_k}\bm{r}_nf_k(\widetilde{h}_{k,n})}
\end{equation}
where $\bm{r}=(\bm{r}_1,\dots,\bm{r}_{N_k})$ is a sequence of i.i.d. Rademacher random variables, namely, $\mathbb{P}(\bm{r}_n=+1)=\mathbb{P}(\bm{r}_n=-1)=1/2$. The quantity $\widetilde{\mathcal{R}}(\mathcal{F}_k(\widetilde{h}^{(k)}))$ measures on average how well the function class $\mathcal{F}_k$ correlates with random noise, and thus describes the richness of $\mathcal{F}_k$. The \emph{expected} Rademacher complexity at agent $k$ is defined as
\begin{equation}\label{eq: individual expected Rademacher complexity}
    \rho_k\triangleq\mathbb{E}_{\widetilde{{\h}}^{(k)}}\widetilde{\mathcal{R}}\Big(\mathcal{F}_k(\widetilde{\bm{h}}^{(k)})\Big),
\end{equation} 
which is the expectation of $\widetilde{\mathcal{R}}(\mathcal{F}_k(\widetilde{h}^{(k)}))$ over all sample sets of size $N_k$ drawn according to \eqref{eq: training sample distribution}. We also define the (expected) \emph{network Rademacher complexity} according to 
\begin{equation}\label{eq: expected network Rademacher complexity}
    \rho\triangleq\sum_{k=1}^Kp_k\rho_k.
\end{equation}
With the above definitions, we can establish a lower bound on the probability of consistent training in \eqref{eq: probability of consistent training}.
\begin{theorem}[\textbf{Probability of consistent training}]\label{theorem: 1}
    Assume $\rho<\mathdutchcal{E}_\Phi(\mathsf{R}^o,0)$, where $\mathdutchcal{E}_\Phi(\mathsf{R}^o,0)$ is defined by \eqref{eq: d_0^star} in \ref{appendix: Lemma}. Under Assumptions \ref{assump: risk function}--\ref{assump: network} and condition \eqref{eq: R^o condition}, the probability of consistent training $P_c$ in \eqref{eq: probability of consistent training} is lower bounded by
    \begin{equation}\label{eq: P_c}
        P_c\geq 1-2\exp\left\{-\frac{8N_{\max}}{\alpha^2\beta^2}\Big(\mathdutchcal{E}_\Phi(\mathsf{R}^o,0)-\rho\Big)^2\right\},
    \end{equation}
    where
    \begin{equation}\label{eq: alpha}
        N_{\max}\triangleq\max_k N_k,\quad \alpha\triangleq\sum_{k=1}^Kp_k\frac{N_{\max}}{N_k},
    \end{equation}
    and $\beta$ is the bound on the function $f_k$ specified in Assumption \ref{assump: bound}.
\end{theorem}
\begin{proof}
    See \ref{appendix: Lemma}.
\end{proof}
\noindent The quantity $\alpha$ is called the \emph{network imbalance penalty}, which quantifies how unequal the numbers of training samples are across different agents. Theorem \ref{theorem: 1} is an extension of the SML consistency result obtained exclusively for the logistic loss $\Phi(x)=\log(1+e^{-x})$ in \cite{bordignon2023learning}, where the expression for $\mathdutchcal{E}_\Phi(\mathsf{R}^o,0)$ was computed explicitly. Deriving the closed form of $\mathdutchcal{E}_{\Phi}(\mathsf{R}^o,0)$ for a general loss function $\Phi$ from Assumption \ref{assump: risk function} is not a trivial task, since it requires us to solve the generic equation \eqref{eq: d_delta^star} in \ref{appendix: Lemma}.
 
The most important implication from Theorem \ref{theorem: 1} is that the probability of consistent training is bounded in an exponential manner if the network Rademacher complexity $\rho$ is smaller than the function $\mathdutchcal{E}_\Phi(\mathsf{R}^o,0)$. According to \eqref{eq: mathdutch_E_R_delta definition} in \ref{appendix: Lemma}, the following relation holds for $\mathdutchcal{E}_\Phi(\mathsf{R}^o,0)$:
\begin{equation}
    \mathdutchcal{E}_\Phi(\mathsf{R}^o,0)\triangleq\frac{d_0^\star}{4}=\frac{\Phi(d_0^\star)-\mathsf{R}^o}{8L_\Phi}
\end{equation}
where $d_0^\star$ is derived by solving  \eqref{eq: d_delta^star}. Using a first-order approximation for the function $\Phi$ around $0$, we have
\begin{equation}
    \Phi(d_0^\star)\approx \Phi(0)+\Phi^\prime(0)d_0^\star,
\quad\text{and}\quad
    d_0^\star\approx \frac{\Phi(0)-\mathsf{R}^o}{2L_\Phi - \Phi^\prime(0)}.
\end{equation}
As already discussed for \eqref{eq: R^o condition}, the risk $\Phi(0)$ corresponds to the uninformative classifier, i.e., to the case of classification by randomly guessing. Therefore, the quantity $\Phi(0)-\mathsf{R}^o$ captures the difficulty of the binary classification task for the network. The closer the target risk $\mathsf{R}^o$ is to the uninformative risk $\Phi(0)$, the smaller the value of $\mathdutchcal{E}_\Phi(\mathsf{R}^o,0)$ and consequently, the more restricted (due to the assumption of $\rho<\mathdutchcal{E}_\Phi(\mathsf{R}^o,0)$) the complexity of the classifier structure will be. Therefore, expression \eqref{eq: P_c} reveals a remarkable interplay between the inherent difficulty of the classification problem (quantified by $\mathdutchcal{E}_\Phi(\mathsf{R}^o,0)$) and the complexity of the classifier structure (quantified by $\rho$).

The \emph{probability of error} achieved by the SML strategy, denoted by $P_e$, is defined as the probability of inconsistent learning:
\begin{equation}\label{eq: P_e definition}
    P_e\triangleq\mathbb{P}\Big(\gammab_{0}\widehat{\lambdab}_{\rm asym}\leq 0\Big)
\end{equation}
where the randomness stems from both the training phase (i.e., the training set) and the prediction phase (i.e., the true label $\gammab_{0}$). We next show that an upper bound for $P_e$ can be obtained from the probability of consistent training $P_c$. Denoting the prior for $\gammab_{0}$ by $\mathbb{P}(\gammab_{0})$, expression \eqref{eq: convergence under different hypotheses} yields
\begin{align}
    \nonumber
    P_e\overset{\eqref{eq: P_e definition}}{=}&\;\mathbb{P}(\gammab_0=+1)\mathbb{P}\Big(\widehat{\lambdab}_{\rm asym}\leq 0 | \gammab_{0}=+1\Big)+\mathbb{P}(\gammab_0=-1)\mathbb{P}\Big(\widehat{\lambdab}_{\rm asym}\geq 0 | \gammab_{0}=-1\Big)\\
    \nonumber
    \overset{\eqref{eq: convergence under different hypotheses}}{=}&\; \mathbb{P}(+1)\mathbb{P}\Big(\mu^{+}(\widetilde{\bm{f}})-\widetilde{\bm{\mu}}(\widetilde{\bm{f}})\leq 0\Big)+\mathbb{P}(-1)\mathbb{P}\Big(\mu^{-}(\widetilde{\bm{f}})-\widetilde{\bm{\mu}}(\widetilde{\bm{f}}) \geq 0\Big)\\
    \label{eq: P_e upper bound by 1-P_c}
    {\leq}\;&\;\mathbb{P}(+1)(1-P_c)+\mathbb{P}(-1)(1-P_c)=1-P_c
\end{align}
where the inequality is due to the definition \eqref{eq: probability of consistent training} of $P_c$, which guarantees:
\begin{equation}
    1-P_c=\mathbb{P}\Big(\Big\{ \mu^{+}(\widetilde{\bm{f}})-\widetilde{\bm{\mu}}(\widetilde{\bm{f}})\leq 0 \Big\} \cup \Big\{\mu^{-}(\widetilde{\bm{f}})-\widetilde{\bm{\mu}}(\widetilde{\bm{f}}) \geq 0 \Big\} \Big).
\end{equation}
Therefore, in the asymptotic regime where the number of observations grows indefinitely during the prediction phase, the probability of error attained by the SML strategy is upper bounded by $1-P_c$, which can be further refined using Theorem \ref{theorem: 1}.

\section{Non-asymptotic performance for statistical classification tasks}
\label{sec: statistical classification}

In this section, we analyze the probability of error for the classification task assuming a \emph{finite number} of observations. In this setting, the agents try to identify the true label $\gammab_0$ given a finite sequence of streaming feature vectors: 
\begin{equation}
    \h_{k,1},\h_{k,2},\dots,\h_{k,S}
\end{equation} 
where $S$ is the size of the stream of unlabeled samples in the prediction phase. It is notable that when $S$ tends to infinity, we will recover the classical social learning task \cite{bordignon2023learning}, whose probability of error can be upper bounded by $1-P_c$ as shown in \eqref{eq: P_e upper bound by 1-P_c}. Since $S$ is finite, a non-asymptotic performance analysis for the distributed learning rule \eqref{eq: sl learning rule} is needed. To this end, we characterize the \emph{instantaneous} probability of error for each agent $k$. Let $\gammab_{k,i}$ denote agent $k$'s decision at time $i$. It is obvious that $\gammab_{k,i}$ depends on the observations received by the network up to time $i$. Without loss of generality, we assume a uniform initial belief for the distributed learning rule \eqref{eq: sl learning rule} so that $\lambdab_{k,0}=0,\forall k\in\mathcal{K}$.

At each time $i$, the agents make a decision according to the sign of their log-belief ratios, i.e., $\gammab_{k,i}\triangleq\sign(\lambdab_{k,i})$. Hence, a misclassification occurs at agent $k$ if $\lambdab_{k,i}$ and $\gammab_{0}$ have different signs. Let $P_{k,i}^{e}$ denote the instantaneous probability of error associated with agent $k$ at time $i$:
\begin{equation}\label{eq: instantaneous probability of error}
    P_{k,i}^{e}\triangleq\mathbb{P}\big(\gammab_{0}\lambdab_{k,i}\leq 0\big).
\end{equation}
Before analyzing \eqref{eq: instantaneous probability of error}, we need to introduce another condition related to the performance of the training phase, which we will refer to as the \emph{$\delta$-margin consistent training} condition. This condition is similar to the consistent training condition in \eqref{eq: consistent training} established for the infinite-sample classification problem, which connects the performance of the training phase (i.e., $P_c$) with that of the prediction phase (i.e., $P_e$) through \eqref{eq: P_e upper bound by 1-P_c}. We will show that the $\delta$-margin consistent training condition plays a key role in our subsequent non-asymptotic performance analysis. Formally, the $\delta$-margin consistent training condition, given a set of trained models $\{\widetilde{f}_k\}$, is expressed by:
\begin{equation}\label{eq: strong condition for consistent training}
    \mu^{+}(\widetilde{f})>\widetilde{ \mu}({\widetilde{f}})+\delta\quad\text{and}\quad \mu^{-}(\widetilde{f})<\widetilde{ \mu}(\widetilde{f})-\delta
\end{equation} 
where $\delta\geq 0$\label{notation: margin} is a non-negative constant. In view of \eqref{eq: convergence under different hypotheses}, the parameter $\delta$ describes the distance between the asymptotic decision statistic $\widehat{\lambdab}_{\rm asym}$ and the decision boundary $0$, which we will refer to as the \emph{decision margin}. To avoid confusion, we note that $\delta$ is a design parameter in condition \eqref{eq: strong condition for consistent training}. However, for each specified learning setup, we can evaluate the value of $\delta$ that is achieved by the SML strategy by calculating the value of the variables $\mu^{+}(\widetilde{f})$, $\mu^{-}(\widetilde{f})$, and $\widetilde{ \mu}(\widetilde{f})$.

It is clear that for any positive $\delta>0$, the $\delta$-margin consistent training condition \eqref{eq: strong condition for consistent training} is stronger than the consistent training condition given by \eqref{eq: consistent training}.  Let $P_{c,\delta}$ denote the probability of $\delta$-margin consistent training:
\begin{equation}\label{eq: probability of delta-margin consistent training}
    P_{c,\delta}\triangleq\mathbb{P}\left( \mu^{+}(\widetilde{\bm{f}})>\widetilde{\bm \mu}({\widetilde{\bm f}})+\delta, \mu^{-}(\widetilde{\bm{f}})<\widetilde{\bm \mu}(\widetilde{\bm f})-\delta \right).
\end{equation}
Here, we use boldface fonts to emphasize the fact that condition \eqref{eq: strong condition for consistent training} depends on the randomness coming from the training phase. Notice that $P_{c,\delta}\leq P_c$ and $P_{c,0}=P_c$. A lower bound on $P_{c,\delta}$ is obtained as follows.
\begin{theorem}[\textbf{Probability of $\delta$-margin consistent training}] \label{theorem: 2}
    Assume that $0\leq\delta<\delta_{\max}$ and $\rho<\mathdutchcal{E}_\Phi(\mathsf{R}^o,\delta)$, where the definitions of $\delta_{\max}$ and $\mathdutchcal{E}_\Phi(\mathsf{R}^o,\delta)$ are respectively given by \eqref{eq: delta_max definition} and \eqref{eq: definition of E(R,delta)} in \ref{appendix: Lemma}. Then, under Assumptions \ref{assump: risk function}--\ref{assump: network} and condition \eqref{eq: R^o condition}, the probability of $\delta$-margin consistent training in \eqref{eq: probability of delta-margin consistent training} is lower bounded by:
    \begin{equation}\label{eq: P_c_delta}
        P_{c,\delta}\geq 1- 2\exp\left\{-\frac{8N_{\max}}{\alpha^2\beta^2}\Big(\mathdutchcal{E}_\Phi({\mathsf{R}^o},\delta)-\rho\Big)^2\right\}.
    \end{equation}	
\end{theorem}
\begin{proof}
    See \ref{appendix: Lemma}.
\end{proof}
\noindent Comparing Theorems \ref{theorem: 1} and \ref{theorem: 2}, we see that the lower bounds on $P_c$ and $P_{c,\delta}$ differ only in the term $\mathdutchcal{E}_\Phi({\mathsf{R}^o},\delta)$. By choosing $\delta=0$, we recover \eqref{eq: P_c} from \eqref{eq: P_c_delta}. According to \eqref{eq: comparison between delta and 0} in \ref{appendix: Lemma}, we have:
\begin{equation}
    \mathdutchcal{E}_\Phi({\mathsf{R}^o},\delta)=\mathdutchcal{E}_\Phi({\mathsf{R}^o},0)+\frac{\Phi(d_\delta^\star)-\Phi(d_0^\star)}{8L_\Phi}
\end{equation}
where $d_\delta^\star$ is the solution to equation \eqref{eq: d_delta^star} given in \ref{appendix: Lemma}. Furthermore, since $d_\delta^\star$ increases with $\delta$ as proved in \ref{appendix: Lemma} and $\Phi$ is non-increasing under Assumption \ref{assump: risk function}, the function $\mathdutchcal{E}_\Phi({\mathsf{R}^o},\delta)$ is also non-increasing with $\delta$. Accordingly, the lower bound on $P_{c,\delta}$ is not increased when a larger decision margin $\delta$ is desired. Moreover, we can carry out a sample complexity analysis for the training phase using \eqref{eq: P_c_delta}, as stated in the forthcoming corollary.
\begin{corollary}[\textbf{Training sample complexity}] 
    Assume that $0\leq\delta<\delta_{\max}$ and $\rho<\mathdutchcal{E}_\Phi(\mathsf{R}^o,\delta)$. For any $\varepsilon_{\rm tr}>0$, the $\delta$-margin consistent training condition \eqref{eq: strong condition for consistent training} holds with probability at least $1-\varepsilon_{\rm tr}$ if the maximum number of training samples across the agents satisfies
    \begin{equation}
        N_{\max}\geq\frac{\alpha^2\beta^2}{8\Big(\mathdutchcal{E}_\Phi({\mathsf{R}^o},\delta)-\rho\Big)^2}\log\left(\frac{2}{\varepsilon_{\rm tr}}\right).
    \end{equation}
\end{corollary}
\begin{proof}
    We obtain the result by setting the right-hand side of \eqref{eq: P_c_delta} to be no smaller than $1-\varepsilon_{\rm tr}$.
\end{proof}
\noindent With the established bound on the $\delta$-margin consistent training condition, we are now able to examine the instantaneous probability of error $P_{k,i}^e$ in \eqref{eq: instantaneous probability of error}. Let $\mathcal{M}_{k,i}$ denote the event of misclassification by agent $k$ at time $i$ and let $\mathcal{C}_\delta$ denote the event of $\delta$-margin consistent training:
\begin{equation}
    \label{eq: event of wrong classification}
    \mathcal{M}_{k,i} \triangleq\big\{ \gammab_{0}\lambdab_{k,i}\leq 0 \big\},\quad
    \mathcal{C}_\delta \triangleq\left\{{\mu^{+}}(\bm{\widetilde{\bm{f}}})-\bm{\widetilde{\bm{\mu}}}(\bm{\widetilde{\bm{f}}})>\delta,{\mu^{-}}(\bm{\widetilde{\bm{f}}})-\bm{\widetilde{\bm{\mu}}}(\bm{\widetilde{\bm{f}}})<-\delta\right\}.	
\end{equation}
According to the law of total probability, the following inequality holds for the probability of error $P_{k,i}^e$:
\begin{align}
    \nonumber
    P_{k,i}^e\overset{\eqref{eq: instantaneous probability of error}}{=}\mathbb{P}(\mathcal{M}_{k,i})&=\mathbb{P}(\mathcal{M}_{k,i}\cap\mathcal{C}_\delta)+\mathbb{P}(\mathcal{M}_{k,i}\cap\overline{\mathcal{C}_\delta})=\mathbb{P}(\mathcal{C}_\delta)\mathbb{P}\left(\mathcal{M}_{k,i}|\mathcal{C}_\delta\right)+\mathbb{P}(\overline{\mathcal{C}_\delta})\mathbb{P}\left(\mathcal{M}_{k,i}|\overline{\mathcal{C}_\delta}\right)\\
    \label{eq: upper bound of p_{k,i}}
    &\leq \mathbb{P}\left(\mathcal{M}_{k,i}|\mathcal{C}_\delta\right)+\mathbb{P}(\overline{\mathcal{C}_\delta}).
\end{align}
Now since $\mathbb{P}(\overline{\mathcal{C}_\delta})=1-P_{c,\delta}$, which can be upper bounded by using Theorem \ref{theorem: 2}, we can derive an upper bound on $P_{k,i}^e$ by characterizing the conditional probability $\mathbb{P}(\mathcal{M}_{k,i}|\mathcal{C}_\delta)$. 
\begin{theorem}[\textbf{Statistical classification error}]\label{theorem: 3}
    Assume that the agents perform the social learning protocol \eqref{eq: sl learning rule} during the prediction phase, then under the $\delta$-margin consistent training condition \eqref{eq: strong condition for consistent training}, we have
    \begin{equation}\label{eq: P_{k,i} under consistent training}
        \mathbb{P}(\mathcal{M}_{k,i}|\mathcal{C}_\delta)\leq \exp\left\{-\frac{\left(\delta i -\kappa\right)^2}{2\beta^2 i}\right\}
    \end{equation}
    for all $i\geq\frac{\kappa}{\delta}$, where 
    \begin{equation}\label{eq: definition of kappa}
        \kappa\triangleq\frac{8\beta\log K}{1-\sigma},
    \end{equation}
    and $0\leq\sigma<1$ is the second largest-magnitude eigenvalue of the combination matrix $A$. Therefore, supposing that the same assumptions as those in Theorem \ref{theorem: 2} hold, for any sequence of observations of size $S\geq\frac{\kappa}{\delta}$, the probability of classification error at each agent $k$, denoted by $P_{k,S}^e$, is upper bounded by
    \begin{align}
        \label{eq: P_e for statistical classification}
        P_{k,S}^e\leq &\;2\exp\left\{-\frac{8N_{\max}}{\alpha^2\beta^2}\Big(\mathdutchcal{E}_\Phi({\mathsf{R}^o},\delta)-\rho\Big)^2\right\}+\exp\left\{-\frac{\left(\delta S -\kappa\right)^2}{2\beta^2 S}\right\}.
    \end{align}  
\end{theorem}
\begin{proof}
    See \ref{appendix: Theorem 1}.
\end{proof}
\noindent We analyze next how we can relate the bound on consistent training from Theorem \ref{theorem: 1} (for the \emph{infinite}-sample classification) to the result of Theorem \ref{theorem: 3} (for the \emph{finite}-sample classification). As the number of observations $S$ grows, the second term on the right-hand side of \eqref{eq: P_e for statistical classification} vanishes. Then, the upper bound for $P_{k,S}^e$ approaches the first term in \eqref{eq: P_e for statistical classification}, which is an upper bound on $1-P_{c,\delta}$ from Theorem \ref{theorem: 2}. Since $P_{c,\delta}\leq P_c$, in view of \eqref{eq: P_e upper bound by 1-P_c}, $1-P_{c,\delta}$ is also an upper bound on the probability of error $P_e$ for the social learning task (i.e., the infinite-sample case). By letting $\delta\to 0$, we recover the upper bound $1-P_c$ established in \eqref{eq: P_e upper bound by 1-P_c}. Furthermore, according to \eqref{eq: P_e for statistical classification}, it is expected that the decay rate of $P_{k,S}^e$ w.r.t. $S$ will be faster when a larger decision margin $\delta$ is achieved in the training phase. Equation \eqref{eq: P_{k,i} under consistent training} provides a good estimate for the sample complexity in the prediction phase, as summarized in following Corollary \ref{corollary: sample complexity}.
\begin{corollary}[\textbf{Testing sample complexity}]\label{corollary: sample complexity}
    Assume the $\delta$-margin consistent training condition is achieved in the training phase. For any $\varepsilon_{\rm ts}>0$, each agent gets $P_{k,S}^e\leq \varepsilon_{\rm ts}$ if the number of observations during prediction satisfies
    \begin{equation}\label{eq: sample complexity}
        S\geq\frac{1}{\delta^2}\Big({\beta\sqrt{\beta^2(\log\varepsilon_{\rm ts})^2-2\kappa\delta\log\varepsilon_{\rm ts}}-\beta^2\log\varepsilon_{\rm ts}+\kappa\delta}\Big).
    \end{equation}
\end{corollary}
\begin{proof}
    We obtain the result by setting the right-hand side of \eqref{eq: P_{k,i} under consistent training} to be no greater than $\varepsilon_{\rm ts}$.
\end{proof}
\noindent Theorem \ref{theorem: 3} demonstrates the effect of different parameters on the probability of error for statistical classification. In particular, both the second largest-magnitude eigenvalue $\sigma$ and the Perron eigenvector $p$ are determined by the combination policy $A$. The decision margin $\delta$ plays a similar role to that of the minimum weighted Kullback-Leibler divergence in the social learning problem when the likelihood models are known accurately \cite{lalitha2018social,nedic2017fast,salami2017social,shahrampour2016distributed,matta2020Interplay,kayaalp2022arithmetic}. We include more discussion on the effect of $A$ and $\delta$ in the supplementary material \cite{supp}.

\section{Numerical Simulations}
\label{sec:simulations}	

In the simulations, we implement the SML strategy on binary image classification tasks built from two datasets: FashionMNIST \cite{xiao2017/online} and CIFAR10 \cite{krizhevsky2009learning}. We consider a network of 9 spatially distributed agents, where each agent observes a part of the image and they are connected through a strongly-connected communication network with the topology depicted in Fig. \ref{fig: net-a}. We also assume a self-loop for each agent (not shown in Fig. \ref{fig: net-a}). The uniform averaging rule is employed for constructing the combination policy $A$ \cite{sayed2014adaptation}. 
For both datasets, the details on the local classifier structure can be found in \cite{supp}.

\begin{figure}
\centering
\subfloat[]{\includegraphics[width=.3\linewidth]{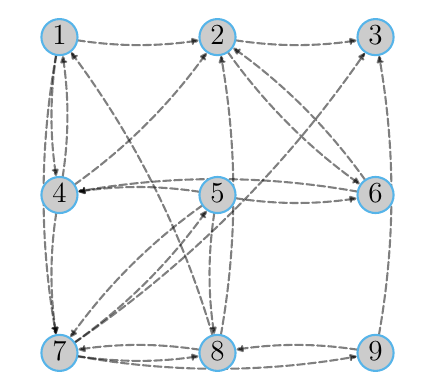}%
\label{fig: net-a}}
\hfil
\subfloat[]{\includegraphics[width=.5\linewidth]{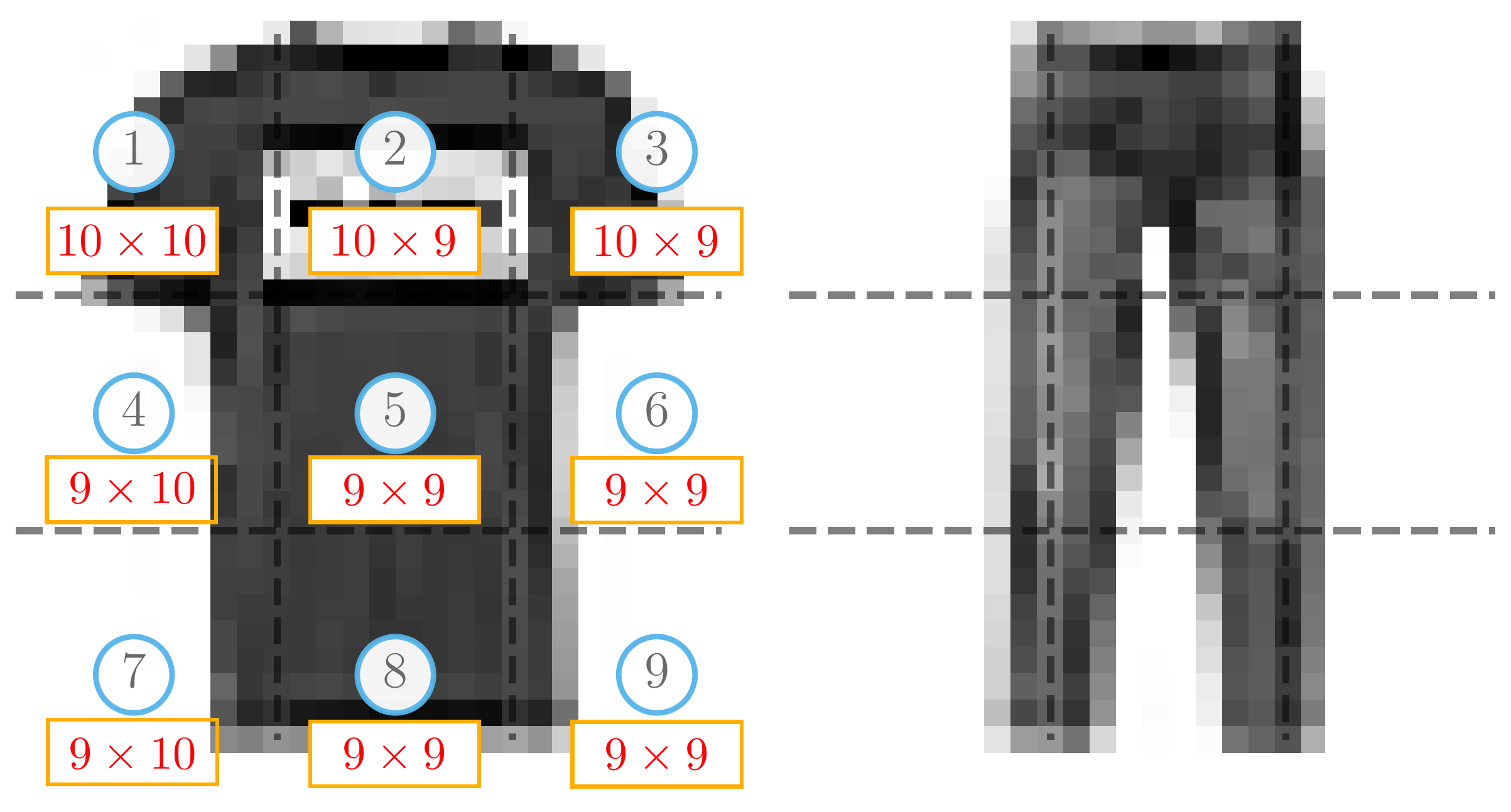}%
\label{fig: net-b}}
\caption{(a) Topology of the communication network involving $9$ agents. (b) Observation map for the $9$ agents in the binary classification tasks constructed from the FashionMNIST dataset.}
\label{fig: net}
\end{figure}

\subsection{FashionMNIST dataset}
Each image of this dataset contains 784 ($28\times 28$) pixels. These pixels are assumed to be distributed as evenly as possible among the 9 agents in the network. We build a binary classification problem to distinguish ``T-shirt'' (labeled with $+1$) from ``trouser'' (labeled with $-1$). The size of the partial image observed by each agent is shown in Fig. \ref{fig: net-b}. In the training phase, each agent trains a local classifier represented by a feedforward neural network with one hidden layer of $15$ neurons (see \cite{supp} for more details). This simple structure is employed in order to better visualize the probability of error curves. To illustrate the $\delta$-margin consistent training condition, we study different sizes of training sets. For simplicity, we assume an identical training size for all agents, i.e., $N_k=N_0, \forall k\in\mathcal{K}$. Given the value of $N_0$, a balanced training set is generated by randomly sampling from the FashionMNIST dataset. For each selected training set, the training is run using mini-batch iterates of 10 samples over 30 epochs. We employ the logistic loss in our simulations. The Adam optimizer \cite{Adam} with learning rate  0.0001 is adopted. 

In Fig. \ref{fig: performance-margin}, we show the decision margins $\delta$ achieved under different training set sizes $N_0$, where the results are averaged over 200 different randomly generated training sets for each $N_0$. It can be observed that on average, the achieved $\delta$ increases as $N_0$ grows. This indicates that with more training samples, a better learning condition (with better trained classifiers) is obtained for the prediction phase. Next, we investigate the learning performance of the SML strategy. In our simulations, the true state $\gammab_0$ in the prediction phase is set to be ``T-shirt''. For each training set considered in Fig. \ref{fig: performance-margin}, we conduct 5000 Monte Carlo runs of statistical classification based on the trained classifiers associated with this training set and obtain the averaged result. The simulation result for a specified training set size $N_0$ is then estimated empirically from the associated 200 training sets. 
\begin{figure}
\centering
\subfloat[]{\includegraphics[width=.46\linewidth]{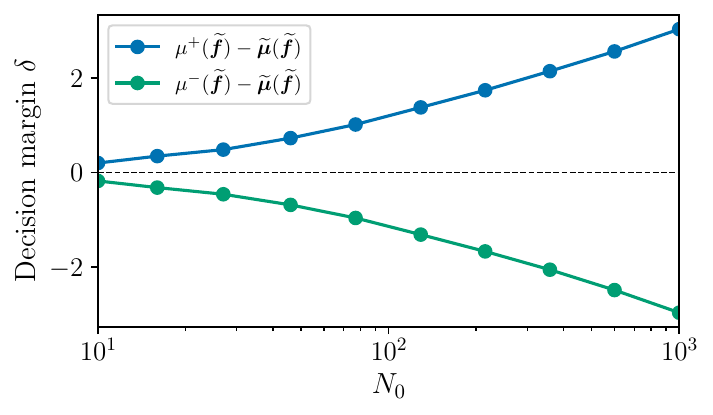}%
\label{fig: performance-margin}}
\hfil
\subfloat[]{\includegraphics[width=.49\linewidth]{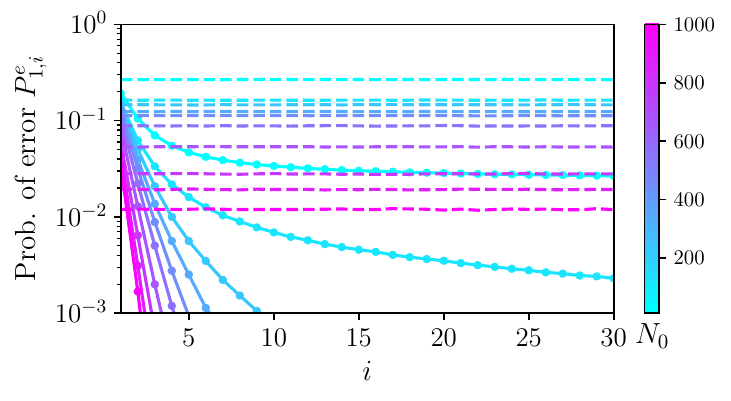}%
\label{fig: performance-sl}}
\caption{(FashionMNIST) (a) Decision margin under different training set sizes $N_0$. (b) Evolution of the probability of error within the SML strategy (solid lines with circles) and AdaBoost (dashed lines) over the prediction samples $i$ for different $N_0$.}
\label{fig: FashionMNIST}
\end{figure}

For performance comparison, we introduce the classical AdaBoost strategy where the $9$ agents are trained sequentially and their \emph{hard} decisions are combined according to the accuracy attained by each local classifier \cite{bartlett1998boosting}. We remark that the statistical classification involving \emph{distributed} and \emph{streaming} observations is not a well developed topic in the literature. The well-known AdaBoost strategy is taken here as a representative classifier to show the importance of developing new strategies for this classification task. Different from the SML strategy that aggregates the information over both time and space following the rule \eqref{eq: sl learning rule}, AdaBoost ignores the temporal dependence in the true label among the unlabeled samples. 

In Fig. \ref{fig: performance-sl}, we show the learning performance of agent $1$ in the SML strategy and that of the centralized decision in AdaBoost. The evolution of the instantaneous probability of error $P_{1,i}^e$ under different training set sizes $N_0$ is presented. We can see that for all $N_0$, $P_{1,i}^e$ decreases over time $i$ (i.e., the number of observations collected so far). For larger $N_0$, the decaying is almost exponential and the decay rate is positively correlated to the achieved decision margin $\delta$ provided by Fig. \ref{fig: performance-margin}, which is consistent with \eqref{eq: P_e for statistical classification} in Theorem \ref{theorem: 3}. In contrast to it, due to the lack of information aggregation over time, the probability of error attained by the centralized AdaBoost strategy remains invariant as the number of observations grows. Particularly, for each $N_0$, the SML strategy outperforms the AdaBoost when more than two unlabeled samples are collected. 

\subsection{CIFAR10 dataset}
We consider the binary classification problem to distinguish ``cats" from ``dogs" within this dataset. A convolutional neural network composed of two convolutional layers is employed by the agents (see \cite{supp} for more details). In our simulations, the training is run using a mini-batch of 128 samples over 100 epochs. 

In Fig. \ref{fig: performance-margin-cifar10}, we present the decision margins $\delta$ achieved by the SML strategy under different training set sizes $N_0$. Except for the case $N_0=10000$ where all training samples of the selected classes are utilized, 200 different training sets are generated randomly for each other $N_0$ to obtain the averaged result. It is obvious from Fig. \ref{fig: performance-margin-cifar10} that in general, a better $\delta$ can be achieved with more labeled samples in the training phase. Whereas, the decision margins attained in Fig. \ref{fig: performance-margin-cifar10} are much smaller than those shown in Fig. \ref{fig: performance-margin}. This implies that for the specified classifiers, distinguishing cats from dogs within the CIFAR10 dataset is much harder than distinguishing T-shirts from trousers within the FashionMNIST dataset.
\begin{figure}
\centering
\subfloat[]{\includegraphics[width=.46\linewidth]{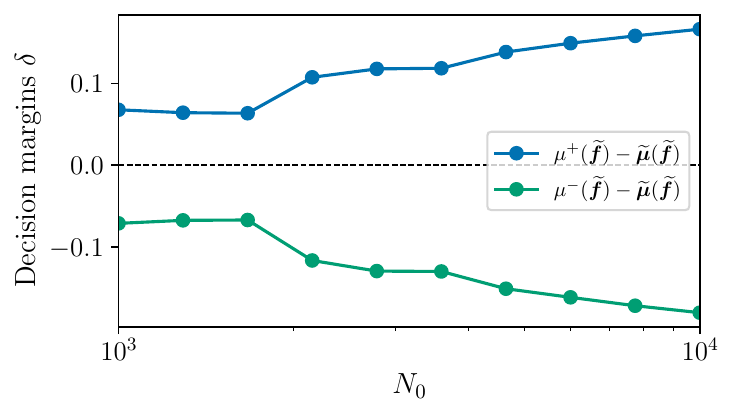}%
\label{fig: performance-margin-cifar10}}
\hfil
\subfloat[]{\includegraphics[width=.49\linewidth]{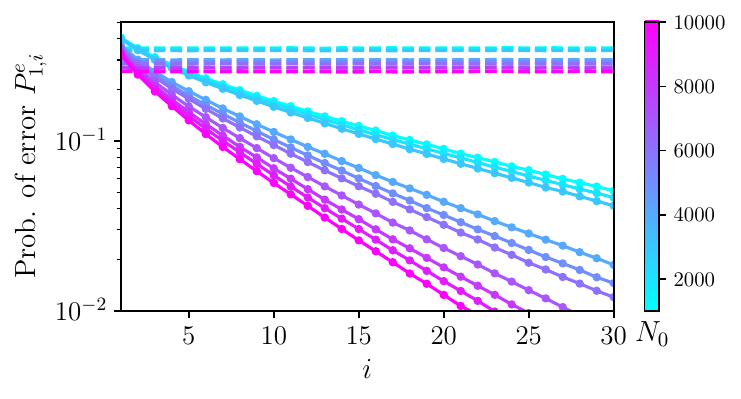}%
\label{fig: performance-sl-cifar10}}
\caption{(CIFAR10) (a) Decision margin under different training set sizes $N_0$. (b)Evolution of the probability of error within the SML strategy (solid lines with circles) and AdaBoost (dashed lines) over the prediction samples $i$ for different $N_0$.}
\label{fig: CIFAR10}
\end{figure}

For the statistical classification task where the true class is ``cat'', the instantaneous probability of error associated with agent $1$ within the SML strategy and that of the centralized AdaBoost strategy are presented in Fig. \ref{fig: performance-sl-cifar10}. As the number of observations grows, an exponential decay of the misclassification error within the SML strategy is observed for each $N_0$ in the simulations. Meanwhile, the decay rate becomes larger when the decision margin is increased with more training samples available for the training phase. Compared to AdaBoost, the SML strategy exhibits a significant advantage in prediction accuracy by effectively leveraging the temporal dependence in the label of the streaming unlabeled samples. Specially, agent $1$ is able to make a better decision within two iterations following the social learning rule \eqref{eq: sl learning rule}.

\section{Concluding remarks}
\label{sec: conclusion}
This paper studies the learning performance of the social machine learning strategy, which is a fully data-driven distributed decision-making architecture. For statistical classification tasks involving a limited number of samples in the prediction phase, we derive an upper bound on the probability of error. Our results extend the analysis in \cite{bordignon2023learning}, which investigated the classification error when the number of observations grows indefinitely. There are many interesting open questions regarding the performance of the social machine learning strategy. As indicated in \cite{bordignon2023learning}, this strategy can be formulated to address multi-class classification tasks. However, the theoretical guarantees in this case are more involved than the binary case. Specifically, the $\delta$-margin consistent training condition becomes more intricate, and the analytical methodology for the lower bound (Theorem \ref{theorem: 2}) needs to be developed. The performance analysis for multi-class classification tasks will be considered in future work.

\appendix

\section{Proof of Theorems \ref{theorem: 1} and \ref{theorem: 2}}
\label{appendix: Lemma}

Since the condition \eqref{eq: consistent training} for consistent training corresponds to the case $\delta=0$ in the $\delta$-margin consistent training condition \eqref{eq: strong condition for consistent training}, we will focus on the proof of Theorem \ref{theorem: 2} in the following. First, we recall some important results on the estimation errors of empirical risk minimization and the training mean, which are established in \cite{bordignon2023learning}.
\begin{inlemma}[\textbf{Theorem 3 in \cite{bordignon2023learning}}]\label{lemma: theorem 3}
    Under Assumptions \ref{assump: risk function}--\ref{assump: network}, we have the following two results. First,
    \begin{equation}\label{eq: Lemma A eq1}
        \mathbb{P}\left(\sup_{f\in\mathcal{F}}\abs{\widetilde{\bm{R}}_{\rm emp}(f)-\widetilde{R}(f)}\geq x\right)\leq \exp\left\{ -\frac{N_{\max}(x-4L_\Phi\rho)^2}{2\alpha^2L_\Phi^2\beta^2} \right\},
    \end{equation}
    for any $x>4L_\Phi\rho$, where $\mathcal{F}\triangleq \mathcal{F}_1\times\mathcal{F}_2\cdots\times\mathcal{F}_K$, $\widetilde{\bm{R}}_{\rm emp}(f)$ and $\widetilde{R}(f)$ are the weighted network averages for the empirical and expected risks \eqref{eq: empirical risk} and \eqref{eq: expected risk of f_k}  defined as follows: 
\begin{equation}\label{eq: network average for the expected risk}
    \widetilde{\bm{R}}_{\rm emp}(f)\triangleq\sum_{k=1}^Kp_k\widetilde{\bm{R}}_{k,{\rm emp}}(f_k),\quad \widetilde{R}(f)\triangleq\sum_{k=1}^Kp_k \widetilde{R}_k(f_k).
\end{equation}
Second,
    \begin{equation}\label{eq: Lemma A eq2}
        \mathbb{P}\left(\sup_{f\in\mathcal{F}}\abs{\widetilde{\bm{\mu}}(f)-\mu(f)}\geq x\right)\leq \exp\left\{ - \frac{N_{\max}(x-4\rho)^2}{2\alpha^2\beta^2}\right\},
    \end{equation}
    for any $x>4\rho$, where 
    \begin{equation}\label{eq: expected training mean}
        \mu(f)=\frac{\mu^{+}(f)+\mu^{-}(f) }{2},\quad \forall f\in\mathcal{F}.
    \end{equation}
\end{inlemma}\qed

\noindent We will also use the following property for strictly monotonic functions and their inverse functions \cite{binmore1982mathematical}.
\begin{property}[\textbf{Inverse of strictly monotone function \cite{binmore1982mathematical}}]\label{Property: inverse function}
    Let $f$ be a real function defined on $I\subseteq\mathbb{R}$ whose image is $J\subseteq\mathbb{R}$. Assume that $f$ is strictly monotone on $I$, then $f$ always has an inverse function $f^{-1}$ that has the same monotonicity as $f$. That is, if $f$ is strictly increasing (or decreasing), then so is $f^{-1}$.\qed
\end{property}
\noindent In order to establish Theorem \ref{theorem: 2}, we upper bound the probability $1- P_{c,\delta}$ with the following inequality:
\begin{align}
    \nonumber
    1- P_{c,\delta}
    &=\mathbb{P}\left( \left\{\mu^{+}(\widetilde{\bm{f}})\leq\widetilde{\bm \mu}({\widetilde{\bm f}})+\delta \right\} \cup \left\{ \mu^{-}(\widetilde{\bm{f}})\geq\widetilde{\bm \mu}(\widetilde{\bm f})-\delta \right\} \right)\\
    \nonumber
    &=\mathbb{P}\left(\left\{\mu^{+}(\widetilde{\bm{f}})-\mu(\widetilde{\bm{f}})\leq\widetilde{\bm{\mu}}(\widetilde{\bm{f}})-\mu(\widetilde{\bm{f}})+\delta\right\}\cup\left\{ \mu^{-}(\widetilde{\bm{f}})-\mu(\widetilde{\bm{f}})\geq\widetilde{\bm \mu}(\widetilde{\bm f})-\mu(\widetilde{\bm{f}})-\delta\right\}\right)\\
    \nonumber
    &\overset{\textnormal{(a)}}{=}\mathbb{P}\left(\abs{\widetilde{\bm \mu}(\widetilde{\bm{f}})-\mu(\widetilde{\bm{f}})}\geq\frac{\mu^{+}(\widetilde{\bm{f}})-\mu^{-}(\widetilde{\bm{f}}) }{2}-\delta \right)\\ \label{eq: proof of Lemma 1}
    &\overset{\textnormal{(b)}}{\leq}\mathbb{P}\left(\abs{\widetilde{\bm \mu}(\widetilde{\bm{f}})-\mu(\widetilde{\bm{f}})}\geq d-\delta\right)+\mathbb{P}\left(\frac{\mu^{+}(\widetilde{\bm{f}})-\mu^{-}(\widetilde{\bm{f}}) }{2}\leq d \right)
\end{align}
for any $d>\delta$, where (a) comes from the definition of $\mu(\widetilde{\bm{f}})$ in \eqref{eq: expected training mean} and (b) holds due to the law of total probability. Using \eqref{eq: Lemma A eq2} from Lemma \ref{lemma: theorem 3}, the first term of \eqref{eq: proof of Lemma 1} satisfies 
\begin{align}
    \label{eq: Lemma 1-first term}
    \mathbb{P}\left(\abs{\widetilde{\bm \mu}(\widetilde{\bm{f}})-\mu(\widetilde{\bm{f}})}\geq d-\delta\right)\leq \mathbb{P}\left(\sup_{f\in\mathcal{F}}\abs{\widetilde{\bm \mu}(f)-\mu(f)}\geq d-\delta\right)\leq\exp\left\{ -\frac{N_{\max}(d-\delta-4\rho)^2}{2\alpha^2\beta^2} \right\}
\end{align}
for all $d>4\rho+\delta$. The bound on the second term of \eqref{eq: proof of Lemma 1} can be established based on Assumption \ref{assump: risk function} and Jensen's inequality. Given the function $f_k\in\mathcal{F}_k$ for each $k\in\mathcal{K}$, the network average of the expected risks evaluated on the training samples $(\widetilde{\h}_{k,n},\widetilde{\gammab}_{k,n})$ is bounded as follows:
\begin{align}
    \nonumber
    \widetilde{R}(f)&\overset{\eqref{eq: network average for the expected risk}}{=}\sum_{k=1}^Kp_k\mathbb{E}_{(\widetilde{\h}_{k,n},\widetilde{\gammab}_{k,n})}\Phi\left(\widetilde{\gammab}_{k,n}f_k(\widetilde{\h}_{k,n})\right)\overset{\textnormal{(a)}}{\geq}\sum_{k=1}^Kp_k\Phi\left(\mathbb{E}_{(\widetilde{\h}_{k,n},\widetilde{\gammab}_{k,n})}\widetilde{\gammab}_{k,n}f_k(\widetilde{\h}_{k,n})\right)\\
    \nonumber
    &\;\overset{\textnormal{(b)}}{\geq} \Phi\left(\sum_{k=1}^Kp_k\mathbb{E}_{(\widetilde{\h}_{k,n},\widetilde{\gammab}_{k,n})}\widetilde{\gammab}_{k,n} f_k(\widetilde{\h}_{k,n}) \right)\overset{\textnormal{(c)}}{=}\Phi\left(\frac{1}{2}\sum_{k=1}^Kp_k\mathbb{E}_{+1}f_k(\widetilde{\h}_{k,n}) - \frac{1}{2}\sum_{k=1}^Kp_k\mathbb{E}_{-1}f_k(\widetilde{\h}_{k,n})\right)\\
    \label{eq: Jensen inequality}
    &\;\overset{\textnormal{(d)}}{=} \Phi \left(\frac{\mu^{+}(f)-\mu^{-}(f)}{2}\right),
\end{align}
where in (a) and (b) we use Jensen's inequality with the convex function $\Phi$. In (c), we use the uniform prior assumption on the training samples, and in (d) we use the definition of conditional means in \eqref{eq: expectation +1}--\eqref{eq: network expectation}. We recall that the notation $\mathbb{E}_{\gamma}$ denotes the expectation operator associated with the likelihood models $L_k(\cdot|\gamma)$ for $\gamma\in\Gamma$. From \eqref{eq: Jensen inequality} and the non-increasing property of $\Phi$, we know that for any given function $f\in\mathcal{F}$:
\begin{equation}
    \frac{\mu^{+}(f)-\mu^{-}(f)}{2}\leq d\Rightarrow \widetilde{R}(f)\geq \Phi(d).
\end{equation}
Replacing the generic $f$ with the trained function $\widetilde{\bm f}$ obtained from the empirical risk minimization \eqref{eq: trained f_k}, we have
\begin{align}\label{eq: Lemma1-second term}
    \mathbb{P}\left(\frac{\mu^{+}(\widetilde{\bm{f}})-\mu^{-}(\widetilde{\bm{f}}) }{2}\leq d \right)\leq\mathbb{P}\left(\widetilde{R}(\widetilde{\bm f})\geq\Phi(d)\right).
\end{align}
The probability on the right-hand side can be bounded using the uniform bound on the estimation error for the empirical risk \eqref{eq: Lemma A eq1}. First, we develop the following inequality: 
\begin{align}
    \nonumber
    \widetilde{R}(\widetilde{\bm{f}})-\mathsf{R}^o&\overset{\textnormal{(a)}}{=}\widetilde{R}(\widetilde{\bm{f}}) -\inf_{f\in\mathcal{F}} \widetilde{R}(f)=\widetilde{R}(\widetilde{\bm{f}})-\widetilde{\bm{R}}_{\rm emp}(\widetilde{\bm{f}})+\widetilde{\bm{R}}_{\rm emp}(\widetilde{\bm{f}})-\inf_{f\in\mathcal{F}} \widetilde{R}(f)\\
    \nonumber
    &=\widetilde{R}(\widetilde{\bm{f}})-\widetilde{\bm{R}}_{\rm emp}(\widetilde{\bm{f}})+\sup_{f\in\mathcal{F}}\left(\widetilde{\bm{R}}_{\rm emp}(\widetilde{\bm{f}})-\widetilde{R}(f) \right)\overset{\textnormal{(b)}}{\leq} \widetilde{R}(\widetilde{\bm{f}})-\widetilde{\bm{R}}_{\rm emp}(\widetilde{\bm{f}})+\sup_{f\in\mathcal{F}}\left(  \widetilde{\bm{R}}_{\rm emp}(f)-\widetilde{R}(f) \right)\\
    \label{eq: 2sup}
    &\leq2\sup_{f\in\mathcal{F}}\abs{\widetilde{\bm{R}}_{\rm emp}(f)-\widetilde{R}(f)}
\end{align}
where (a) follows directly from the definition of $\mathsf{R}^o$ in \eqref{eq: target risk}, while (b) is based on the definitions of $\widetilde{\bm f}_k$ and $\widetilde{\bm{R}}_{\rm emp}(f)$ in \eqref{eq: approximate logit function f_k} and \eqref{eq: network average for the expected risk} respectively, which ensure that $\widetilde{\bm{R}}_{\rm emp}(\widetilde{\bm{f}})\leq\widetilde{\bm{R}}_{\rm emp}(f)$. Therefore, one gets
\begin{align}
    \nonumber
    \mathbb{P}\left(\widetilde{R}(\widetilde{\bm f})\geq\Phi(d)\right)&=\mathbb{P}\left(\widetilde{R}(\widetilde{\bm f})-\mathsf{R}^o\geq\Phi(d)-\mathsf{R}^o\right)\overset{\textnormal{(a)}}{\leq}\mathbb{P}\left(\sup_{f\in\mathcal{F}}\abs{\widetilde{\bm{R}}_{\rm emp}(f)-\widetilde{R}(f)}\geq\frac{\Phi(d)-\mathsf{R}^o}{2} \right)\\
    \label{eq: Lemma1-second term 2}
    &\overset{\textnormal{(b)}}{\leq} \exp\left\{ -\frac{N_{\max}\Big(\frac{\Phi(d)-\mathsf{R}^o}{2}-4L_\Phi\rho\Big)^2}{2\alpha^2L_\Phi^2\beta^2} \right\}
\end{align}
for any $d$ such that $\frac{\Phi(d)-\mathsf{R}^o}{2}>4L_\Phi\rho$, where (a) comes from \eqref{eq: 2sup} and (b) is derived using \eqref{eq: Lemma A eq1} in Lemma \ref{lemma: theorem 3}.
According to \eqref{eq: proof of Lemma 1}, by combining \eqref{eq: Lemma 1-first term}, \eqref{eq: Lemma1-second term}, and \eqref{eq: Lemma1-second term 2}, we have
\begin{align}
    \nonumber
    P_{c,\delta}&\geq  1-\mathbb{P}\left(\abs{\widetilde{\bm \mu}(\widetilde{\bm{f}})-\mu(\widetilde{\bm{f}})}\geq d-\delta\right) -\mathbb{P}\left(\frac{\mu^{+}(\widetilde{\bm{f}})-\mu^{-}(\widetilde{\bm{f}}) }{2}\leq d \right)\\
    \label{eq: Lemma1-bound with two terms}
    &\geq1- \exp\left\{-\frac{N_{\max}(d-\delta-4\rho)^2}{2\alpha^2\beta^2}\right\}-\exp\left\{-\frac{N_{\max}\Big(\frac{\Phi(d)-\mathsf{R}^o}{2}-4L_\Phi\rho\Big)^2}{2\alpha^2L_\Phi^2\beta^2}\right\}
\end{align}
for any $d$ contained in the following interval: $d\in(4\rho+\delta, d_{\max})$, where $d_{\max}$ is defined by
\begin{equation}\label{eq: interval of d}
    d_{\max}\triangleq\sup\left\{x:\Phi(x)=\mathsf{R}^o+8L_\Phi\rho \right\}.
\end{equation}
If the loss function $\Phi$ is strictly monotonic, then from Property \ref{Property: inverse function}, $\Phi$ has an inverse function $\Phi^{-1}$ and $d_{\max}$ can be written as $d_{\max}=\Phi^{-1}(\mathsf{R}^o+8L_\Phi\rho)$. Observe that the leading coefficients of the exponents in \eqref{eq: Lemma1-bound with two terms} are equal when
\begin{equation}
    d-\delta-4\rho = \frac{\Phi(d)-\mathsf{R}^o}{2L_\Phi}-4\rho,
\end{equation}
which yields $\Phi(d)-\mathsf{R}^o-2L_\Phi(d-\delta)=0$.
Define $g(d)$ as the following function of variable $d$:
\begin{equation}\label{eq: g(d)}
    g(d)\triangleq d-\frac{\Phi(d)-\mathsf{R}^o}{2L_\Phi}.
\end{equation} 
For a given decision margin $\delta\geq 0$, let $d_\delta^\star$ denote a solution to the equation $g(d)=\delta$, i.e.,
\begin{equation}\label{eq: d_delta^star}
    g(d_{\delta}^\star)=\delta.
\end{equation}
If we can prove that $d_\delta^\star>\delta$, then from \eqref{eq: Lemma1-bound with two terms}, the probability of $\delta$-margin consistent training $P_{c,\delta}$ can be lower bounded as
\begin{equation}\label{eq: P_{c,delta} bound in proof}
    P_{c,\delta}\geq 1 - 2\exp\left\{-\frac{8N_{\max}\Big(\frac{d_\delta^\star-\delta}{4}-\rho\Big)^2}{\alpha^2\beta^2}\right\}
\end{equation}
when the Rademacher complexity $\rho$ satisfies $\rho<\frac{d_\delta^\star-\delta}{4}$. Let
\begin{equation}\label{eq: definition of E(R,delta)}
    \mathdutchcal{E}_\Phi(\mathsf{R}^o,\delta)\triangleq \frac{d_\delta^\star-\delta}{4}.
\end{equation}
We will show that such constant $d_\delta^\star$ exists when the margin $\delta$ is small, i.e., $\mathdutchcal{E}_\Phi(\mathsf{R}^o,\delta)>0$ for a small $\delta$. 

First, we notice that due to the non-increasing property of $\Phi$ (Assumption \ref{assump: risk function}), $g(d)$ is a strictly increasing function of $d$. This implies that for any $\delta>0$, the solution $d_\delta^\star$ specified in \eqref{eq: d_delta^star} is unique. Furthermore, $\Phi$ is differentiable at $0$ and $\Phi^\prime(0)<0$ under Assumption \ref{assump: risk function}. With the condition $\mathsf{R}^o<\Phi(0)$ in \eqref{eq: R^o condition}, we have
\begin{equation}
    g(0)=-\frac{\Phi(0)-\mathsf{R}^o}{2L_\Phi}<0.
\end{equation}
Since $g(d)$ is increasing, we know that the equation $g(d)=\delta$ has a positive solution for the case $\delta=0$. That is, $\exists d_0^\star>0$ such that $g(d_0^\star)=0$. Due to the strict monotonicity of $g$, we know that according to Property \ref{Property: inverse function}, there is an inverse function $g^{-1}$ which is strictly increasing. Therefore, for any $\delta>0$, the equation $g(d)=\delta$ has a unique positive solution $d_\delta^\star=g^{-1}(\delta)$. Moreover, $d_\delta^\star$ increases with a larger margin $\delta$.  Let
\begin{equation}
    d_{\mathsf{R}}\triangleq\inf \left\{x:\Phi(x)=\mathsf{R}^o \right\}
\end{equation}
be the infimum of the set of $x$ corresponding to the target risk $\mathsf{R}^o$. Particularly, if the loss function $\Phi$ is strictly decreasing, we have $d_{\mathsf{R}}=\Phi^{-1}(\mathsf{R}^o)$. It is clear from $\eqref{eq: R^o condition}$ that $d_{\mathsf{R}}>0$. By definition of $d_{\delta}^\star$ in \eqref{eq: d_delta^star}, we get from \eqref{eq: g(d)} and \eqref{eq: definition of E(R,delta)} that
\begin{equation}\label{eq: mathdutch_E_R_delta definition}
    \mathdutchcal{E}_\Phi(\mathsf{R}^o,\delta)= \frac{d_\delta^\star-\delta}{4}=\frac{\Phi(d_\delta^\star)-\mathsf{R}^o}{8L_\Phi}.
\end{equation}
Therefore, $\mathdutchcal{E}_\Phi(\mathsf{R}^o,\delta)>0$ if, and only if, $\Phi(d_\delta^\star)>\mathsf{R}^o$. In other words, the lower bound provided by \eqref{eq: P_{c,delta} bound in proof} is meaningful if, and only if, $d_\delta^\star<d_{\mathsf{R}}$.
Since $d_{\delta}^{\star}$ is increasing with $\delta$, the decision margin $\delta$ must be selected to be smaller than a constant $\delta_{\max}$ whose definition is 
\begin{equation} \label{eq: delta_max definition}
    \delta_{\max}\triangleq\sup\big\{\delta\geq 0: g(d)=\delta \text{ has a solution } d_\delta^\star<d_{\mathsf{R}}\big\}.
\end{equation}
The existence of the constant $\delta_{\max}$ is guaranteed by the strict monotonicity of function $g$. Therefore, \eqref{eq: P_{c,delta} bound in proof} holds for any $\delta<\delta_{\max}$ and $\rho<\mathdutchcal{E}_\Phi(\mathsf{R}^o,\delta)$. This proves \eqref{eq: P_c_delta} in Theorem \ref{theorem: 2}. 

Since our proof does not require $\delta$ to be strictly greater than 0, the above analysis applies also to the case $\delta=0$, i.e., the case of consistent training \eqref{eq: consistent training}. Let $\delta=0$ in \eqref{eq: definition of E(R,delta)}, we have
\begin{equation}\label{eq: d_0^star}
    \mathdutchcal{E}_\Phi(\mathsf{R}^o,0)=\frac{d_0^\star}{4}.
\end{equation}
In addition, we can establish the following relation for $\mathdutchcal{E}_\Phi(\mathsf{R}^o,\delta)$ under consistent training and $\delta$-margin consistent training conditions:
\begin{align}
    \label{eq: comparison between delta and 0}
    \mathdutchcal{E}_\Phi(\mathsf{R}^o,\delta)&\overset{\text{(a)}}{=}\frac{\Phi(d_\delta^\star)-\mathsf{R}^o}{8L_\Phi}=\frac{\Phi(d_\delta^\star)-\Phi(d_0^\star)}{8L_\Phi}+\frac{\Phi(d_0^\star)-\mathsf{R}^o}{8L_\Phi}\overset{\text{(b)}}{\leq} \frac{\Phi(d_0^\star)-\mathsf{R}^o}{8L_\Phi}=\mathdutchcal{E}_\Phi(\mathsf{R}^o,0)
\end{align}
where (a) follows from \eqref{eq: mathdutch_E_R_delta definition} and (b) is due to $d_{\delta}^\star\geq d_0^\star$ and the non-increasing property of function $\Phi$.

\section{Proof of Theorem \ref{theorem: 3}}
\label{appendix: Theorem 1}

We introduce McDiarmid's inequality and the convergence of matrix powers in the following Lemmas.

\begin{inlemma}[\textbf{McDiarmid's inequality}] \label{lemma: McDiarmid's inequality}
    Let $\bm{x}$ represent a sequence of independent random variables $\bm{x}_n$, with $n=1,2,\dots,N$ and $\bm{x}_n\in\mathcal{X}_n$ for all $n$. Suppose that the function $g$: $\prod_{n=1}^{N}\mathcal{X}_n\mapsto \mathbb{R}$ satisfies for every $n=1,2,\dots,N$:
    \begin{equation}\label{eq: McDiarmid inequality-bound}
        \abs{g(x)-g(\widehat{x})}\leq b_n
    \end{equation}
    whenever the sequences $x$ and $\widehat{x}$ differ only in the $n$-th component. Then, for any $\epsilon>0$:
    \begin{align}
        \label{eq: McDiarmid inequality-greater}
        \mathbb{P}\left(g(\bm{x})-\mathbb{E} g(\bm{x})\geq \epsilon\right)\leq\exp(-\frac{2\epsilon^2}{\sum_{n=1}^N b_n^2}),\quad
        \mathbb{P}\left( g(\bm{x})-\mathbb{E} g(\bm{x})\leq - \epsilon\right)\leq\exp(-\frac{2\epsilon^2}{\sum_{n=1}^N b_n^2}).
    \end{align}
\end{inlemma}\qed

\begin{inlemma}[\textbf{Convergence of matrix powers \cite{nedic2017fast}}]\label{lemma: convergence of matrix power}
    Consider a strongly-connected network of $K$ agents and a left-stochastic combination policy $A$. Then, for any $t$ and for any agent $k$, the following inequality holds:
    \begin{equation}\label{eq: convergence of matrix power}
        \sum_{\tau=1}^t\sum_{\ell=1}^K\abs{[A^{t-\tau}]_{\ell k}-p_\ell}\leq\frac{4\log K}{1-\sigma}.
    \end{equation}
    where $0\leq \sigma<1$ denotes the second largest-magnitude eigenvalue of $A$.\qed
\end{inlemma}
\noindent Before proving Theorem \ref{theorem: 3}, we note that the approximate logit functions $\widetilde{\bm{f}}_k$ and the corresponding trained classifiers $\widetilde{\bm{c}}_{k}$ are random w.r.t. the training samples $(\widetilde{\h}_{k,n},\widetilde{\gammab}_{k,n})$. Since the training phase is independent of the prediction phase within the SML framework, the randomness stemming from the training phase can be ``frozen'' when we develop our analysis for the prediction phase. Particularly, for any observation $\h_{k,i}$ in the prediction phase, both $\widetilde{\bm{f}}_k(\h_{k,i})$ and $\widetilde{\bm{c}}_{k}(\h_{k,i})$ are deterministic values since the expressions of $\widetilde{\bm{f}}_k$ and $\widetilde{\bm{c}}_{k}$ have been specified in the training phase. To eliminate any potential ambiguity, in this proof, we will use normal fonts for random variables that are independent of the prediction phase. For example, we will use the notation $\widetilde{{f}}_k$ and $\widetilde{{c}}_{k}$ instead of $\widetilde{\bm{f}}_k$ and $\widetilde{\bm{c}}_{k}$ throughout this proof. However, we keep the symbol $\sim$ on top of variables related to the training phase. Our proof proceeds as follows.

 We begin with the case where the true state of the statistical classification task is $+1$, namely, $\gammab_{0}=+1$. According to \eqref{eq: sl learning rule}, the log-belief ratio of agent $k$ at time $i$ is expressed by
\begin{equation}
    \bm{\lambda}_{k,i}\overset{\eqref{eq: sl learning rule}}{=}\sum_{\ell=1}^{K} a_{\ell k}\left(\lambdab_{\ell,i-1}+\widetilde{{c}}_{\ell}({\h}_{\ell,i})\right)
    =\sum_{\ell=1}^K[A^i]_{\ell k}\bm{\lambda}_{\ell,0}+\sum_{\tau=1}^i\sum_{\ell=1}^K [A^{i+1-\tau}]_{\ell k}\widetilde{{c}}_{\ell}(\h_{\ell,\tau}).
\end{equation}
Under the uniform initial belief condition, i.e., $\bm{\lambda}_{\ell,0}=0$ for all $\ell\in\mathcal{K}$, we get
\begin{equation}	\label{eq: instantaneous log-belief ratio}
    \bm{\lambda}_{k,i}=\sum_{\tau=1}^i\sum_{\ell=1}^K [A^{i+1-\tau}]_{\ell k}\widetilde{{c}}_{\ell}(\h_{\ell,\tau}).
\end{equation}
First, we show that the expectation of $\bm{\lambda}_{k,i}$ is lower bounded by the upcoming equation \eqref{eq: lower bound of the expectation}. Taking the expectation of \eqref{eq: instantaneous log-belief ratio} w.r.t. the historical observations received by the network until time $i$, we have
\begin{align}
    \nonumber
    \mathbb{E}_{+1}{\bm{\lambda}_{k,i}}&=\mathbb{E}_{+1}\left[\sum_{\tau=1}^i\sum_{\ell=1}^K [A^{i+1-\tau}]_{\ell k}\widetilde{{c}}_{\ell}(\h_{\ell,\tau}) \right]\overset{\text{(a)}}{=}\sum_{\tau=1}^i\sum_{\ell=1}^K [A^{i+1-\tau}]_{\ell k}\mathbb{E}_{+1}\big[\widetilde{{c}}_{\ell}(\h_{\ell,\tau})\big]\\
    \nonumber
    &\overset{\text{(b)}}{=}\sum_{\tau=1}^i\sum_{\ell=1}^K [A^{i+1-\tau}]_{\ell k}\big(\mu_\ell^{+}(\widetilde{{f}}_\ell)-\widetilde{{\mu}}_\ell(\widetilde{{f}}_{\ell})\big)\\
    \label{eq: expectation of log-belief ratio}
    &\overset{\text{(c)}}{=}\sum_{\tau=1}^i\sum_{\ell=1}^K \big([A^{i+1-\tau}]_{\ell k}-p_\ell\big)\big(\mu_\ell^{+}(\widetilde{{f}}_{\ell})-\widetilde{{\mu}}_\ell(\widetilde{{f}}_{\ell})\big)+i\big(\mu^{+}(\widetilde{{f}})-\widetilde{{\mu}}(\widetilde{{f}})\big)
\end{align}
where in (a) we use the independence of local observations over time conditioned on the true state $\gammab_{0}$, and in (b) we use the assumption of $\gammab_{0}=+1$. In (c) we use the definitions of $\mu^{+}(\widetilde{{f}})$ and $\widetilde{{\mu}}(\widetilde{{f}})$ in \eqref{eq: network expectation} and \eqref{eq: network training mean}. 
With the bound of $f_k$ specified in Assumption \ref{assump: bound}, we get $|\mu_\ell^{+}(\widetilde{{f}}_\ell)-\widetilde{{\mu}}_\ell(\widetilde{{f}}_\ell)|\leq 2\beta$. In view of the convergence of matrix powers \eqref{eq: convergence of matrix power} in Lemma \ref{lemma: convergence of matrix power}, we have
\begin{align} 
    \nonumber
    \abs{\sum_{\tau=1}^i\sum_{\ell=1}^K \big([A^{i+1-\tau}]_{\ell k}-p_\ell\big)\big(\mu_\ell^{+}(\widetilde{{f}}_\ell)-\widetilde{{\mu}}_\ell(\widetilde{{f}}_\ell)\big)}&\overset{\textnormal{(a)}}{\leq}\sum_{\tau=1}^i\sum_{\ell=1}^K  \abs{[A^{i+1-\tau}]_{\ell k}-p_\ell}\cdot\max_{\ell\in\mathcal{K}}\abs{\mu_\ell^{+}(\widetilde{{f}}_\ell)-\widetilde{{\mu}}_\ell(\widetilde{{f}}_\ell)}\\
    &\leq\frac{4\log K}{1-\sigma}\cdot 2\beta=\frac{8\beta\log K}{1-\sigma}=\kappa
\end{align}
where (a) follows from H{\"o}lder's inequality \cite{sayed2022inference}, and $\kappa$ is defined in \eqref{eq: definition of kappa}.
Therefore, the expectation $\mathbb{E}_{+1}{\bm{\lambda}_{k,i}}$ in \eqref{eq: expectation of log-belief ratio} is lower bounded by
\begin{equation}\label{eq: lower bound of the expectation}
    \mathbb{E}_{+1}{\bm{\lambda}_{k,i}}\geq i(\mu^{+}(\widetilde{{f}})-\widetilde{{\mu}}(\widetilde{{f}}))-\kappa.
\end{equation}
Next, we show that the boundedness of $f_k$ also ensures that the log-belief ratio $\bm{\lambda}_{k,i}$, as a function involving the historical observations, satisfies the bounded difference condition \eqref{eq: McDiarmid inequality-bound}. We first recall the notation $\bm{\mathsf{h}}_{i}$ for the collection of observations at time $i$ in \eqref{eq: random vector of observations}. It is worth noting that $\bm{\mathsf{h}}_{i}$ is independent over time. Following \eqref{eq: instantaneous log-belief ratio}, we can view the log-belief ratio of agent $k$ at time $i$ as a function of the random vectors $\bm{\mathsf{h}}_1$, $\bm{\mathsf{h}}_2$, \dots, $\bm{\mathsf{h}}_{i}$. Let us denote
\begin{equation}
    \bm{\mathsf{H}}_{i}\triangleq\text{col}\big\{\bm{\mathsf{h}}_1,{\bm{\mathsf{h}}}_2,\dots,{\bm{\mathsf{h}}}_i \big\}=\text{col}\big\{\h_{k,\tau}, \forall k\in\mathcal{K}, \forall 1\leq \tau\leq i\big\}
\end{equation}
as the sequence of observations received by the network up to time $i$. For any given sequence $\bm{\mathsf{H}}_i$, the corresponding log-belief ratio of agent $k$ is represented by $\lambdab_{k,i}(\bm{\mathsf{H}}_i)$. Now let us consider another sequence of observations $\widehat{\bm{\mathsf{H}}}_i$ that differs from $\bm{\mathsf{H}}_i$ only in $\bm{\mathsf{h}}_{\tau}$, i.e., $\widehat{\bm{\mathsf{h}}}_{\tau^\prime}=\bm{\mathsf{h}}_{\tau^\prime}$, $\forall\tau^\prime\neq\tau$. Then, the difference between $\lambdab_{k,i}(\bm{\mathsf{H}}_i)$ and $\lambdab_{k,i}(\widehat{\bm{\mathsf{H}}}_i)$ is bounded as
\begin{align}
    \nonumber
    \abs{\lambdab_{k,i}(\bm{\mathsf{H}}_i)-\lambdab_{k,i}(\widehat{\bm{\mathsf{H}}}_i) }
    &\overset{\text{(a)}}{=}\abs{\sum_{\ell=1}^K [A^{i+1-\tau}]_{\ell k}\widetilde{{c}}_{\ell}(\h_{\ell,\tau})-\sum_{\ell=1}^{K}[A^{i+1-\tau}]_{\ell k}\widetilde{{c}}_{\ell}(\widehat{\h}_{\ell,\tau})}\\
    \nonumber
    &\overset{\text{(b)}}{=}\abs{\sum_{\ell=1}^{K}[A^{i+1-\tau}]_{\ell k}\big(\widetilde{{f}}_{\ell}(\h_{\ell,\tau})- \widetilde{{f}}_{\ell}(\widehat{\h}_{\ell,\tau})\big)}\overset{\text{(c)}}{\leq}\sum_{\ell=1}^{K}[A^{i+1-\tau}]_{\ell k}\abs{\widetilde{{f}}_{\ell}(\h_{\ell,\tau})- \widetilde{{f}}_{\ell}(\widehat{\h}_{\ell,\tau})}\\
    \label{eq: bounded variation of each agent}
    &\overset{\text{(d)}}{\leq} 2\sum_{\ell=1}^{K}[A^{i+1-\tau}]_{\ell k}\beta\overset{\text{(e)}}{=}2\beta.
\end{align}
where (a) is due to \eqref{eq: instantaneous log-belief ratio} and the assumption on $\bm{\mathsf{H}}_i$ and $\widehat{\bm{\mathsf{H}}}_i$, and (b) is due to the definition of $\widetilde{{c}}_\ell$ given in \eqref{eq: trained classifier}. In (c), we use the triangle inequality for absolute values. Inequality (d) follows directly from Assumption \ref{assump: bound}. The last equality (e) holds due to the stochastic matrix $A$. That is, the matrix power of a left-stochastic matrix $A$ still gives a left-stochastic matrix: $\forall m=1,2,\dots$, $\mathbbm{1}^\top A^m=\mathbbm{1}^\top A^{m-1}=\dots=\mathbbm{1}^\top A =\mathbbm{1}^\top$, where $\mathbbm{1}$ is the vector of all ones.
In terms of condition \eqref{eq: McDiarmid inequality-bound} in Lemma \ref{lemma: McDiarmid's inequality}, we now know that $\lambdab_{k,i}$ has a bounded difference $b_\tau$ with respect to the random vector $\bm{\mathsf{h}}_\tau$. Moreover, $b_\tau$ is uniform for all random vectors $\bm{\mathsf{h}}_\tau$, i.e., $b_\tau = 2\beta$, $\forall 1\leq\tau\leq i$. Based on this condition, we apply Lemma \ref{lemma: McDiarmid's inequality} to the log-belief ratio $\bm{\lambda}_{k,i}$ and obtain
\begin{align}
\nonumber
\mathbb{P}\big(\bm{\lambda}_{k,i}\leq0|\gammab_{0}=+1\big)&=\mathbb{P}\big(\bm{\lambda}_{k,i}-\mathbb{E}_{+1}{\bm{\lambda}_{k,i}}\leq-\mathbb{E}_{+1}{\bm{\lambda}_{k,i}}\big)\overset{\textnormal{(a)}}{\leq}\mathbb{P}\Big(\bm{\lambda}_{k,i}-\mathbb{E}_{+1}{\bm{\lambda}_{k,i}}\leq\kappa-i\big(\mu^{+}(\widetilde{{f}})-\widetilde{{\mu}}(\widetilde{{f}})\big) \Big)\\
\label{eq: proof of theorem 1- case +1}
&\overset{\textnormal{(b)}}{\leq}\exp\left\{-\frac{2\left( i\big(\mu^{+}(\widetilde{{f}})-\widetilde{{\mu}}(\widetilde{{f}})\big)-\kappa \right)^2}{\sum_{\tau=1}^{i}b_\tau^2}\right\}{\leq}\exp\left\{-\frac{\left( i\big(\mu^{+}(\widetilde{{f}})-\widetilde{{\mu}}(\widetilde{{f}})\big)-\kappa \right)^2}{2\beta^2i}\right\}
\end{align}
for all $i\geq \frac{\kappa}{\mu^{+}(\widetilde{{f}})-\widetilde{{\mu}}(\widetilde{{f}})}>0$, where in (a) we use the lower bound \eqref{eq: lower bound of the expectation}, and in (b) we use the McDiarmid's inequality \eqref{eq: McDiarmid inequality-greater}. Conditioning on the $\delta$-margin consistent training event $\mathcal{C}_{\delta}$ \eqref{eq: event of wrong classification}, we have that for all $i\geq\frac{\kappa}{\delta}$,
\begin{equation}\label{eq: P_{k,i} conditioned on B}
    \mu^{+}(\widetilde{{f}})-\widetilde{{\mu}}(\widetilde{{f}})>\delta,
\quad\text{and}\quad  \mathbb{P}\left(\bm{\lambda}_{k,i}\leq0|\gammab_{0}=+1,\mathcal{C}_\delta\right)\leq \exp\left\{-\frac{\left( i\delta-\kappa \right)^2}{2\beta^2i}\right\}.
\end{equation}
\noindent The above analysis is discussed for the case $\gammab_0=+1$. For the case $\gammab_{0}=-1$, repeating the steps \eqref{eq: expectation of log-belief ratio}--\eqref{eq: P_{k,i} conditioned on B} yields
\begin{equation}\label{eq: P_{k,i} conditioned on B 2}
    \mu^{-}(\widetilde{{f}})-\widetilde{{\mu}}(\widetilde{{f}})<-\delta,
\quad\text{and}\quad   \mathbb{P}\left(\bm{\lambda}_{k,i}\geq0|\gammab_{0}=-1,\mathcal{C}_\delta\right)\leq \exp\left\{-\frac{\left( i\delta-\kappa \right)^2}{2\beta^2i}\right\}
\end{equation}
for all $i\geq\frac{\kappa}{\delta}$. Hence, the conditional probability $\mathbb{P}\left(\mathcal{M}_{k,i}|\mathcal{C}_\delta\right)$ can be bounded as
\begin{equation}\label{eq: proof of theorem 1-conditional probability}
    \mathbb{P}\left(\mathcal{M}_{k,i}|\mathcal{C}_\delta\right)=\mathbb{P}(+1)\mathbb{P}\left(\lambdab_{k,i}\leq 0|\gammab_{0}=+1,\mathcal{C}_\delta\right)+ \mathbb{P}(-1)\mathbb{P}\left(\lambdab_{k,i}\geq 0|\gammab_{0}=-1,\mathcal{C}_\delta\right)\leq \exp\left\{-\frac{( i\delta-\kappa)^2}{2\beta^2i}\right\}
\end{equation}
for all $i\geq\frac{\kappa}{\delta}$. From \eqref{eq: upper bound of p_{k,i}}, the instantaneous probability of error $P_{k,i}^e$ is upper bounded by $ \mathbb{P}\left(\mathcal{M}_{k,i}|\mathcal{C}_\delta\right) + \mathbb{P}(\overline{\mathcal{C}_\delta})$. The proof of Theorem \ref{theorem: 3} is completed by recalling the upper bound for $\mathbb{P}(\overline{\mathcal{C}_\delta})$ in Theorem \ref{theorem: 2}. 

\bibliographystyle{elsarticle-num}
\bibliography{SML}

\newpage
\pagestyle{empty}
\includepdf[pages=-, pagecommand={}, linktodoc=true]{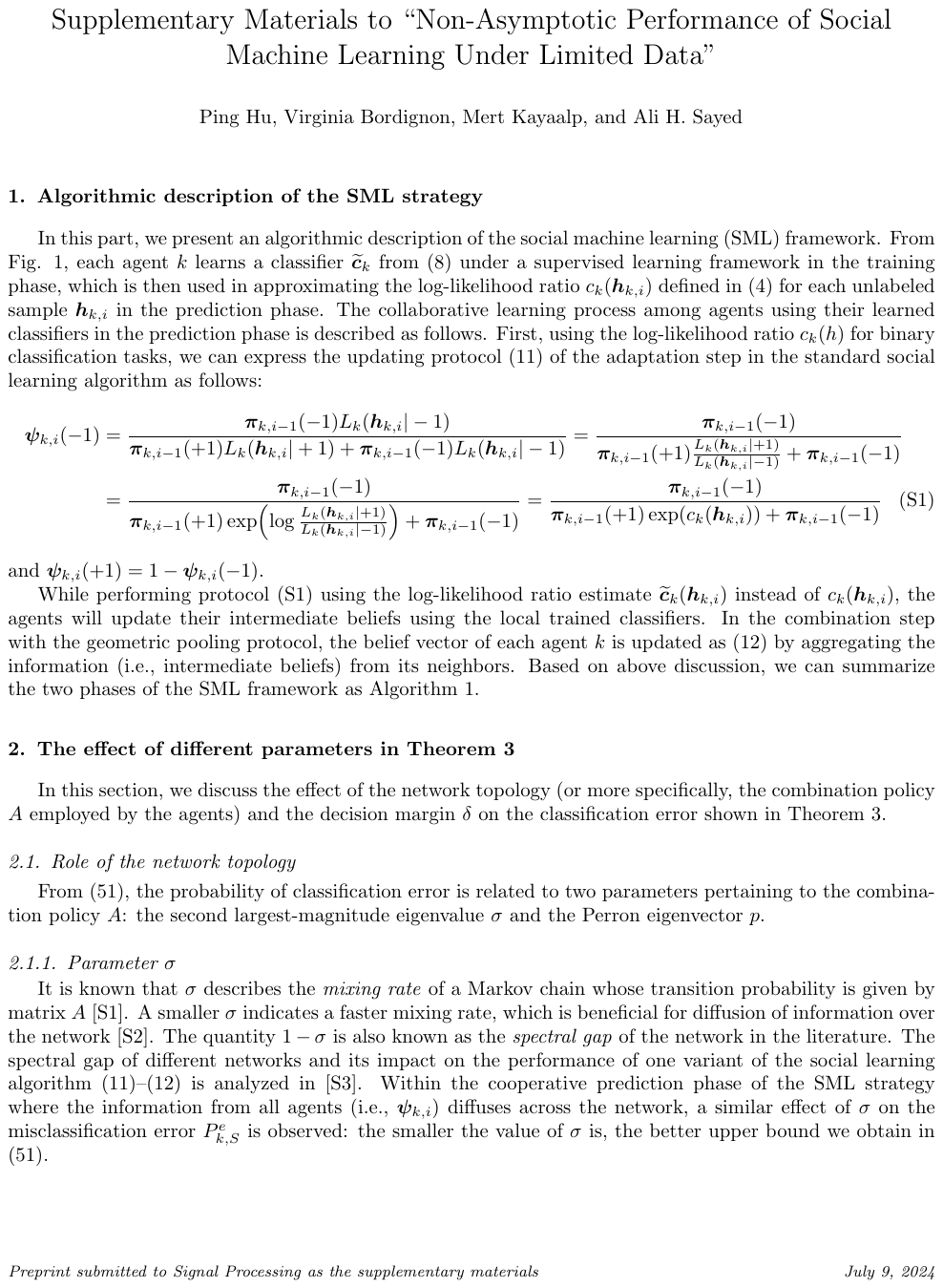}

\end{document}